\documentclass{article}

\usepackage[final]{neurips_2022}

\usepackage{times,epsf,mathtools,xcolor}
\usepackage{wrapfig}
\usepackage{subcaption}
\usepackage{natbib}
\usepackage{graphicx}
\usepackage{booktabs} 
\usepackage{amssymb,amsmath,amsthm, amsfonts}
\newtheorem{theorem}{Theorem}

\usepackage{algorithm}
\usepackage{algorithmic}

\usepackage[utf8]{inputenc} % allow utf-8 input
\usepackage[T1]{fontenc}    % use 8-bit T1 fonts
\usepackage{hyperref}       % hyperlinks
\usepackage{cleveref}
\usepackage{url}            % simple URL typesetting
\usepackage{booktabs}       % professional-quality tables
\usepackage{amsfonts}       % blackboard math symbols
\usepackage{nicefrac}       % compact symbols for 1/2, etc.
\usepackage{microtype}      % microtypography
\usepackage{xcolor}         % colors
\usepackage{csquotes}

\newtheorem{remark}[theorem]{Remark}
\newtheorem{lemma}[theorem]{Lemma}
\newtheorem{definition}[theorem]{Definition}
 
\newenvironment{proof}{{\it Proof:}\quad}{\hfill $\square$ \vskip 12pt}

\newcommand{\N}{\mathbb{N}}
\newcommand{\R}{\mathbb{R}}

\newcommand{\V}{\mathcal{V}}

\newcommand{\E}{\mathcal{E}}
\newcommand{\I}{\mathcal{I}}

\newcommand{\CG}{\mathcal{G}}
\newcommand{\HH}{\mathcal{H}}
\newcommand{\A}{\mathcal{A}}

\newcommand{\wh}{\widehat}

\title{Reduction Algorithms for Persistence Diagrams of Networks: CoralTDA and PrunIT}

\author{%
  Cuneyt G. Akcora \\
  Department of Computer Science\\
  University of Manitoba\\
  \texttt{cuneyt.akcora@umanitoba.ca} \\
  \And
     Yulia R. Gel \\
   Department of Mathematical Sciences \\
   University of Texas at Dallas \\
   National Science Foundation \\
   \texttt{ygl@utdallas.edu} \\
   \AND
   Murat Kantarcioglu \\
   Department of Computer Science \\
   University of Texas at Dallas \\
   \texttt{muratk@utdallas.edu} \\
   \And
   Baris Coskunuzer \\
   Department of Mathematical Sciences \\
   University of Texas at Dallas \\
   \texttt{coskunuz@utdallas.edu} \\
}

\begin{document}

\maketitle
\begin{abstract}

Topological data analysis (TDA) delivers invaluable and complementary information on the intrinsic properties of data inaccessible to conventional methods. However, high computational costs remain the primary roadblock hindering the successful application of TDA in real-world studies, particularly with machine learning on large complex networks.

Indeed, most modern networks such as citation, blockchain, and online social networks often have hundreds of thousands of vertices, making the application of existing TDA methods infeasible. We develop two new, remarkably simple but effective algorithms to compute the exact persistence diagrams of large graphs to address this major TDA limitation. First, we prove that $(k+1)$-core of a graph $\CG$ suffices to compute its $k^{th}$ persistence diagram, $PD_k(\CG)$. Second, we introduce a pruning algorithm for graphs to compute their persistence diagrams by removing the dominated vertices. Our experiments on large networks show that our novel approach can achieve computational gains up to 95\%. 

The developed framework provides the first bridge between the graph theory and TDA, with applications in machine learning of large complex networks. Our implementation is available at \href{https://github.com/cakcora/PersistentHomologyWithCoralPrunit}{github.com/cakcora/PersistentHomologyWithCoralPrunit}.

\end{abstract}

\section{Introduction}

Topological data analysis (TDA) has emerged as powerful machinery in machine learning (ML), allowing us to extract complementary information on the observed objects, especially, from graph-structured data. In particular, TDA has become quite popular in various ML tasks, ranging from bioinformatics~\cite{kovacev2016using, nielson2015topological},  
finance~\cite{leibon2008topological, akcora2019bitcoinheist}, material science~\cite{ichinomiya2017persistent},  
biosurveillance~\cite{segovia2021tlife, chen2022tamp},  network analysis~\cite{sizemore2017classification,  
	carstens2013persistent}, as well as insurance and agriculture~\cite{yuvaraj2021topological, jiang2022learning} (see the literature overviews~\cite{amezquita2020shape, chazal2021introduction} and the TDA applications library~\cite{giunti22}). 
Recently there has emerged a highly active research area that combines the PH machinery with geometric deep learning (GDL) methods~\cite{hofer2017deep, zhao2019learning, horn2021topological}. 

Persistent homology (PH) is a key approach in TDA, allowing us to extract the evolution of subtler patterns in the data shape dynamics at multiple resolution scales, which are not accessible to more conventional, non-topological methods~\cite{carlsson2009topology}. The main idea is to construct a nested sequence of topological spaces (filtration) induced from the data, and record the evolution of topological features in this sequence. In other words, the extracted patterns, or homological features, along with how long such features persist throughout the considered filtration of a scale parameter, convey a critical insight into salient graph characteristics and hidden mechanisms behind system organization. 

PH has been very effective in many graph machine learning tasks, such as graph and node classification~\cite{rieck2019persistent, cai2020understanding, zhao2020persistence, hofer2020graph}, link prediction~\cite{benson2018simplicial, yan2021link} and anomaly detection~\cite{bruillard2016anomaly, ofori2021topological}. 

Nevertheless, while PH has shown promise in various graph learning applications, prohibitive computational costs of PH constrain its wider usage. Indeed, most PH studies are limited to small graphs with a few thousand vertices at most. The problem is that the complexity of the standard PH algorithm is cubic in the number of simplices~\cite{otter2017roadmap}, so one needs to limit homology computations to $0$-th and $1$-th levels only. Computation of higher-level persistence for relatively large graphs can take days or weeks.  

In this paper, we aim to address this fundamental bottleneck in the application of TDA to large networks by introducing two new efficient algorithms which significantly reduce the cost of computing persistence diagrams (PD) for large real-world networks: {\it CoralTDA} and {\it PrunIT}.  

\vskip3pt

\noindent {\bf CoralTDA Algorithm:} Based on our observation that many vertices in large real-world networks have low degrees and do not contribute to PDs in higher dimensions, we developed the CoralTDA algorithm (Theorem \ref{thm:kcores}) where we prove that $(k+1)$-core $\CG^{k+1}$ of a graph $\CG$ is enough to compute the $k^{th}$ PD of the graph, i.e. $PD_k(\CG)=PD_k(\CG^{k+1})$. 	

Using this property, with a much smaller core graph $\CG^{k+1}$, we compute the exact higher persistence diagram $PD_k(\CG)$ losing no information. Our experiments show that even for lower dimensional topological features, such as $k=1$, we reduce the graph order by up to 73\% for some datasets (See Figure~\ref{fig:vertex}). Our findings show that many real-life data sets exhibit nontrivial second and third persistence diagrams, facilitating various classification problems. On the other hand, our reduction reaches 100\% for the third or higher dimensions in several networks, implying that higher PDs are trivial for these datasets.

As a result, our reduction approach improves our understanding of the existence of higher-order dimensional holes and their role in the organization of complex networks.

\vskip3pt

\noindent {\bf PrunIT Algorithm:} We further develop a topologically simple but highly efficient algorithm to facilitate computations of PDs of graphs for any dimension. In particular, for a graph $\CG=(\V,\E)$ and filtration by clique complexes, we show that removing (pruning) a dominated vertex from the graph does not change PDs at any level, provided that the dominated vertex enters the filtration after the dominating vertex (\Cref{thm:reduction}). 

Our experiments indicate that the new algorithm is highly efficient in PD computations of a broad category of large graphs from 100K to 1M vertices, and it can reach 95\% vertex reduction (see Table~\ref{tab:prunitresults}). 

Further, when we combine CoralTDA and PrunIT algorithms, we can significantly reduce the graph sizes for the computation of PDs (Figure \ref{fig:combinedresults}).  

\medskip

{\bf We summarize the key novelty of our contributions as follows:} 

\begin{itemize}
	
	\item We show that the graphs’  ($k+1$)th and higher persistence diagrams only depend on their $k$-cores. 
	\item We introduce a highly effective pruning algorithm that significantly reduces the graph size without changing any persistence diagram of the original graph.
	\item Our experiments in large datasets and large graphs show up to 95\% reduction in the graph size for the computation of persistence diagrams.
	\item With our reduction algorithms, highly successful TDA methods can be applied to very large graphs and large datasets where previously its use was constrained by prohibitive computational costs.
\end{itemize}

\section{Related Work} 
\label{sec:related}

There are mainly two settings in practice where we use PH to obtain a topological fingerprint of a dataset. The first one is the \textit{point cloud setting}, where the dataset comes as a point cloud in an ambient space $\R^n$. Then, we define PH by constructing a sequence of simplicial complexes induced by the pair-wise distances of data points (Vietoris-Rips filtration) and keeping track of the topological changes in this sequence~\cite{zomorodian2005computing, edelsbrunner2010computational}. The second one is \textit{the network setting} where the typical PH construction uses a filtering function on the network. By construction, while the principal identifier to define PH in the point cloud setting is the pair-wise distances of points, the principal identifier in the network setting is the filtering function. Because of this, PH machinery works differently in a network setting, as explained in Section \ref{sec:background}.

There are several works in the point cloud setting to reduce the computational costs and run-time of the persistence diagrams. Malott and Wilsey used the idea of data reduction and data partitioning~\cite{malott2019fast}. Mischaikow and Nanda brought the discrete Morse Theory of geometric topology to the combinatorial setting~\cite{mischaikow2013morse}. In \cite{obayashi2018volume,vcufar2021fast, dey2019persistent, escolar2016optimal}, the authors studied the same problem with different approaches in the point cloud setting.  

While several works improve the run-time of PH in the point cloud setting, only a few of them could reduce the computational costs of persistent homology in the network setting. An idea is to use discrete Morse Theory to capture the topological features occurring during the process~\cite{kannan2019persistent} by applying the techniques developed in~\cite{mischaikow2013morse} to the network setting. 

While the computational complexity of $k^{th}$ persistence diagram (PD) is $\mathcal{O}(n^3)$ where $n$ is the number of $k$-simplices~\cite{otter2017roadmap}, \cite{mischaikow2013morse} achieves $\mathcal{O}(m^2\times n\log{n})$ where $m$ is the number of critical $k$-simplices. With the additional time to find the critical $k$-simplices in each filtration step, the computational complexity  $\mathcal{O}(m^2\times n\log{n})$ is not scalable for very large networks.

\section{Persistent Homology} 
\label{sec:background}

This part provides a background on the theory of persistent homology. Homology $H_k(X)$ is an essential invariant in algebraic topology, which captures the information of the $k$-dimensional holes (connected components, loops, cavities) in a topological space $X$. For example, a connected component in a graph is a zero-dimensional hole, whereas a graph loop is a 1-dimensional hole. Persistent homology is a way to use this invariant to keep track of the changes in a controlled topological space sequence induced by the original space $X$. For basic background on persistent homology, see \cite{edelsbrunner2010computational, dey2022computational}.	

There are several ways to use PH in a network setting, such as power filtration or using different complexes (e.g., Vietoris-Rips, \v{C}ech complexes) to construct the filtration for a given filtering function \cite{aktas2019persistence}. We focus on the most common methods to define PH for graphs: sub/superlevel filtrations obtained by a filtering function and the clique (flag) complexes. Sub/superlevel filtrations are the most common methods because one can inject domain information into the PH process if the chosen filtering function comes from the network domain (e.g., atomic number in protein networks, transaction amount for blockchain networks). Note that our results can be generalized to the persistent homology defined with a filtering function for different complexes.  

Throughout the paper, we use the terms \textit{graph} and \textit{network} interchangeably. Let $\CG$ be a graph with vertex set $\V=\{v_r\}$ and edge set $\E=\{e_{rs}\}$, i.e. $e_{rs}\in \E$ if there is an edge between the vertex $v_r$ and $v_s$ in $\CG$. Let $f:\V\to \R$ be a filtering function defined on the vertices of $\CG$. Let $\I=\{\alpha_i\}$ be a threshold set with $\alpha_0=\min_{v_r \in \V} f(v_r)<\alpha_1<...<\alpha_m=\max_{v_r \in \V} f(v_r)$. For $\alpha_i\in \I$, let $\V_i=\{v_r\in\V\mid f(v_r)\leq \alpha_i\}$. Let $\CG_i$ be the induced subgraph of $\CG$ by $\V_i$, i.e. $\CG_i=(\V_i,\E_i)$ where $\E_i=\{e_{rs}\in \E\mid v_r,v_s\in\V_i\}$. Let $\wh{\CG}_i$ be the clique complex of $\CG_i$. A clique complex is obtained by filling in all the $(k+1)$-complete subgraphs with $k$-simplices. In other words, if the vertices  $\{v_{r_0},v_{r_1},...,v_{r_k}\}\subset \CG_i$ are pairwise connected by an edge in $\CG$, then the clique complex $\wh{\CG}_i$ contains a $k$-simplex $\sigma=[v_{r_0},v_{r_1},...,v_{r_k}]$. This simplicial complex $\wh{\CG}_i$ obtained by filling in all complete subgraphs is called \textit{the clique complex} of $\CG_i$. This construction induces a nested sequence of high dimensional simplicial complexes:

$$\wh{\CG}_0\subset \wh{\CG}_1\subset \wh{\CG}_2\subset ...\subset \wh{\CG}_m.$$

This sequence of simplicial complexes is called \textit{the sublevel filtration} for $\CG$. Superlevel filtrations can be defined similarly by considering the generating sets $\{f(v_r)\geq \alpha_i\}$ instead of $\{f(v_r)\leq \alpha_i\}$ above. Here, $\wh{\CG}_i$ can be taken as the different simplicial complexes induced by $\CG_i$ which gives different types of filtrations~\cite{aktas2019persistence}. After obtaining the filtration, one considers the homology groups $H_k(\wh{\CG}_i)$ of each simplicial complex $\wh{\CG}_i$. The homology group $H_k(X)$ keeps the information of $k$-dimensional topological features in the simplicial complex $X$.

Persistent homology keeps track of the topological changes in the sequence $\{\wh{\CG}_i\}$ by using the homology groups $\{H_k(\wh{\CG}_i)\}$. When a $k$-dimensional hole $\sigma$ (a connected component, loop or cavity) appears in $H_k(\wh{\CG}_i)$, we mark $b_\sigma=\alpha_i$ as its birth time. The feature $\sigma$ can disappear at a later time in  $H_k(\wh{\CG}_j)$ by merging with another feature or by being filled in. Then, we mark $d_\sigma=\alpha_j$ as its death time. Hence, we say that $\sigma$ persists along the interval $[b_\sigma,d_\sigma)$, i.e. $[\alpha_i,\alpha_j)$. The longer the interval ($d_\sigma- b_\sigma$), the more persistent the feature $\sigma$. 

The multi-set $PD_k(\CG,f)=\{(b_\sigma, d_\sigma) \mid \sigma\in H_k(\wh{\CG}_i) \mbox{ for } b_\sigma\leq i<d_\sigma\}$ is called the {\em $k^{th}$ persistence diagram} of $(\CG,f)$ which is the collection of $2$-tuples marking the birth and death times of $k$-dimensional holes $\{\sigma\}$ in $\{\wh{\CG}_i\}$. In particular, $PD_k(\CG,f)$ represents the $k^{th}$ PD of the sublevel filtration, induced by the filtering function $f:\V\to\R$. For brevity, we suppress $f$ and use $PD_k(\CG)$ throughout the text.

\section{CoralTDA Reduction and Higher Persistence Diagrams} 
\label{sec:coresanddiagrams}

A \textit{$k$-core} $\CG^k$ of a graph $\CG$ is the subgraph of $\CG$ obtained by iteratively deleting all vertices (and edges connected to it) with degree less than $k$~\cite{seidman1983network}. In other words, $\CG^k$ is the largest subgraph of $\CG$ where all the vertices have a degree of at least $k$. 

\begin{wrapfigure}{r}{2in} 
	\vspace{-.3cm}
	\begin{center}
		\includegraphics[width=1.9in]{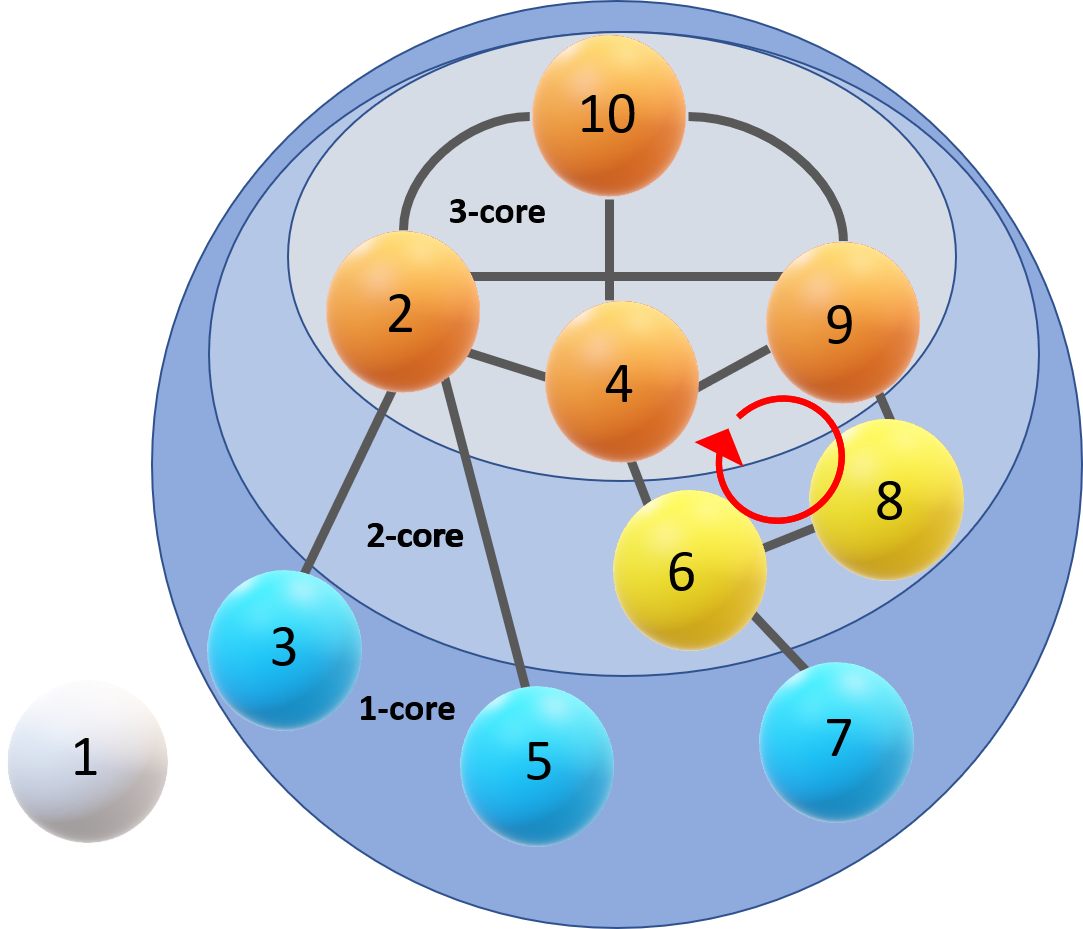}
		\caption{\footnotesize K-core decomposition of a graph of 10 vertices. Vertex 1 has no edges and belongs to the $0th$-core. A one-dimensional hole of vertices 4, 6, 8, and 9 is shown with a red circle.}
		\label{fig:core}
	\end{center}
	\vspace{-.3cm}
\end{wrapfigure} 

Figure~\ref{fig:core} shows a graph with its core structure. Here, vertex $1$ belongs to 0-core as it is disconnected from the graph. Vertex colors indicate shared coreness. When we use vertex degree as the filtering function and allow graph cliques of size three at most, the only one-dimensional hole (shown with the red circle) appears at degree 4 for vertices 4, 6, 8, and 9. Vertices 3, 5, and 7 can only contribute to 0-dimensional holes because their degree is 1. Similarly, 8 can only contribute to 0 and 1-dimensional holes because its degree is 2.

The $k$-core decomposition is a fundamental operation in many areas such as graph similarity matching \cite{nikolentzos2018degeneracy}, graph clustering~\cite{giatsidis2014corecluster}, network visualization~\cite{giatsidis2011evaluating}, anomaly detection~\cite{shanahan2013large} and robustness analysis~\cite{burleson2020k}.

A naïve implementation of $k$-core iteratively deletes vertices whose degree falls below a $k$, until it deletes all vertices from the graph. The implementation has a computational complexity of $\mathcal{O}(m\log{}n)$, where $m$ and $n$ are the number of edges and vertices in the network, respectively. Batagelj and Zaversnik reduce the complexity to $\mathcal{O}(m+n)$ \textquote{by keeping an in-memory array of all possible degree values and keeping track of bin boundaries}~\cite{batagelj2003m}.

\subsection{Relation between \texorpdfstring{$\wh{\CG}_i$}{Gi} and \texorpdfstring{$\wh{\CG}^k_i$}{Gki}} 

Our main idea is to compute high-dimensional persistence features on their associated graph cores. Note that a $k$-clique in graph theory corresponds to a $(k-1)$-simplex in PH; a $k$-clique (complete subgraph of order $k$) in $\CG$ induces a $(k-1)$-simplex in $\wh{\CG}$.

The clique complex $\wh{\CG}$ is a simplicial complex of dimension $K-1$ where $K$ denotes the degeneracy of $\CG$, i.e. $K=\max\{k\mid \CG^k\neq \emptyset\}$. That is, $\wh{\CG}$ contains a $(k-1)$-simplex if and only if its $k$-core  $\CG^k$ is not empty. 

For any $i,k$, we have the $k$-core of $\CG_i$ contained in $\CG_i$ by construction, i.e. $\CG^k_i\subset \CG_i$. This implies that the same holds for their clique complexes, i.e. $\wh{\CG}^k_i\subset \wh{\CG}_i$. On the other hand, if one restricts the original filtering function $f:\V\to \R$ to the vertices $\V^k$ of the $k$-core of $\CG$, we have $f:\V^k\to\R$. By using the same thresholds for $f:\V^k\to\R$, we obtain the filtration $\wh{\CG}^k_0\subset \wh{\CG}^k_1\subset \wh{\CG}^k_2\subset ...\subset \wh{\CG}^k_m$. This will induce the persistence diagram $PD_r(\CG^k)$ for any dimension $r$. 

Since for any $i,k$, $\wh{\CG}^k_i\subset \wh{\CG}_i$, we have the following diagram.
\begin{equation}\label{eqn1}
\begin{array}{ccccccc}
\wh{\CG}^k_0 & \subset  & \wh{\CG}^k_1 & \subset & ... & \subset & \wh{\CG}^k_m  \\
\cap & \ & \cap & \ & \ & \ & \cap \\
\wh{\CG}_0 & \subset  & \wh{\CG}_1 & \subset & ...& \subset & \wh{\CG}_m  
\end{array}
\end{equation}

Notice that for any $j\geq k-1$, if there is a $j$-cycle $\sigma$ living in $C_j(\wh{\CG}_i^k)$, then we have $\sigma\subset C_j(\wh{\CG}_i)$ as $\wh{\CG}^k_i\subset \wh{\CG}_i$. In the following, we show that for these cycles, the converse is also true, and we show the equivalence in the homology level.

\vspace{.2cm}

\begin{remark} \label{remark:restrict_f} \normalfont [Restriction of $f$ to $\V^k$] Notice that the filtering function $f:\V^k\to \R$ on $\V^k$, the vertices of $\CG^k$,  is defined directly by restricting values of $f:\V\to\R$ to the subset $\V^k\subset \V$. In particular, if $f:\V\to\R$ is a function coming from the graph attributes (such as vertex degree), then $f:\V^k\to\R$ may not be the same function coming from the graph attributes induced by the graph $\CG^k$. For example, let $f$ be the degree function on $\V$, the vertices of $\CG$. Then, for any $w\in \V^k$, $f(w)$ is the degree of $w$ in $\CG$, \textit{not its degree in $\CG^k$}. While the $k$-core graph $\CG^k$ changes, we \textit{do not update} the values of $f$ on $\V^k$ according to its attribute definition in $\CG^k$, but we keep the same values in the original function $f:\V\to\R$ for the remaining vertices in $\V^k\subset \V$. In graph terms, this corresponds to computing vertex filtering (activation) values on the original graph but using the edges of the reduced graph to extract simplices.
\end{remark}

\subsection{CoralTDA Reduction} 

Our CoralTDA technique shows that lower degree vertices do not affect higher persistence diagrams, i.e., CoralTDA yields exact results.  Note that in the following result, even though the graph size changes, we keep the same filtering function $f:\V\to\R$ with the original values.  See Remark~\ref{remark:restrict_f} for further details.   We give the proof of the following theorem in Appendix.

\begin{theorem} \label{thm:kcores} Let $\CG$ be an unweighted connected graph. Let $f:\V\to\R$ be a filtering function on $\CG$. Let $PD_k(\CG,f)$ represent the $k^{th}$ persistence diagram for the sublevel filtration of the clique complexes. Let $\wh{\CG}^k$ be the $k$-core of $\CG$. Then, for any $j\geq k$
	$$PD_j(\CG,f)=PD_j(\CG^{k+1},f).$$ 		
\end{theorem}

{\em Outline of the proof:} We show that for any nontrivial $k$-homology class $\sigma$ in the original clique complex $\wh{\CG}$, a generating $k$-cycle $S$ in this homology class also lives in a much smaller subcomplex: the clique complex of the $(k+1)$-core ( $\wh{\CG}^{k+1}$). That is, we prove that any vertex in the $k$-cycle $S$ must have a degree at least $k+1$ where this degree count comes only from the $k$-simplices of $S$, and removing the lower degree vertices from $\CG$ has no effect on the existence of such $S$. We give the proof of the theorem in Appendix.

The above result indicates that $k^{th}$ persistence diagram information can be obtained by only considering the $(k+1)$-core of a graph. CoralTDA is an effective tool for reducing computational costs to compute higher persistence diagrams. See  \Cref{fig:vertex} for reduction results for various datasets.

\begin{figure*}[ht]
	\begin{subfigure}{.24\textwidth}
		\centering
		\includegraphics[width=.99\linewidth]{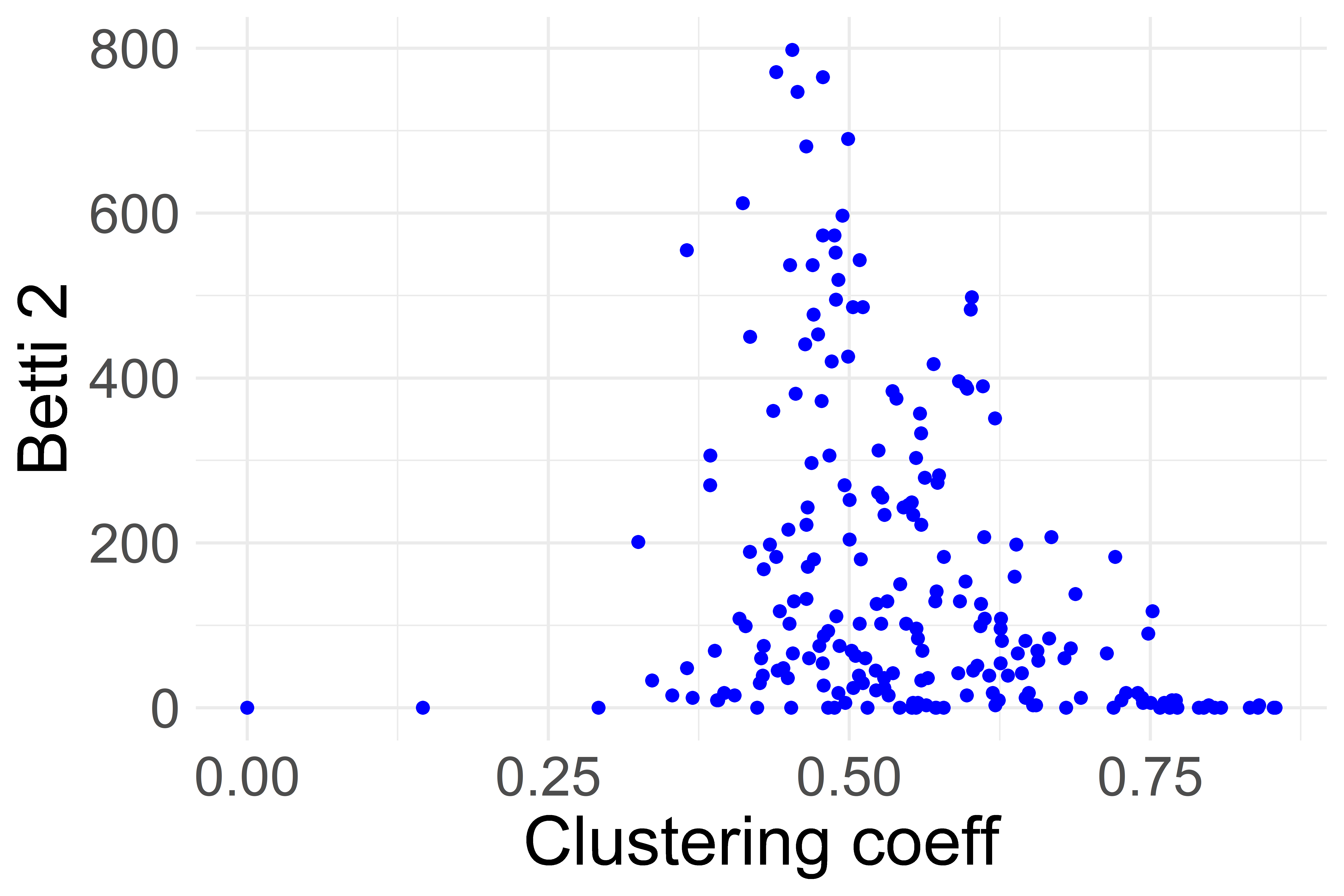}  \end{subfigure}
	\begin{subfigure}{.24\textwidth}
		\centering
		\includegraphics[width=.99\linewidth]{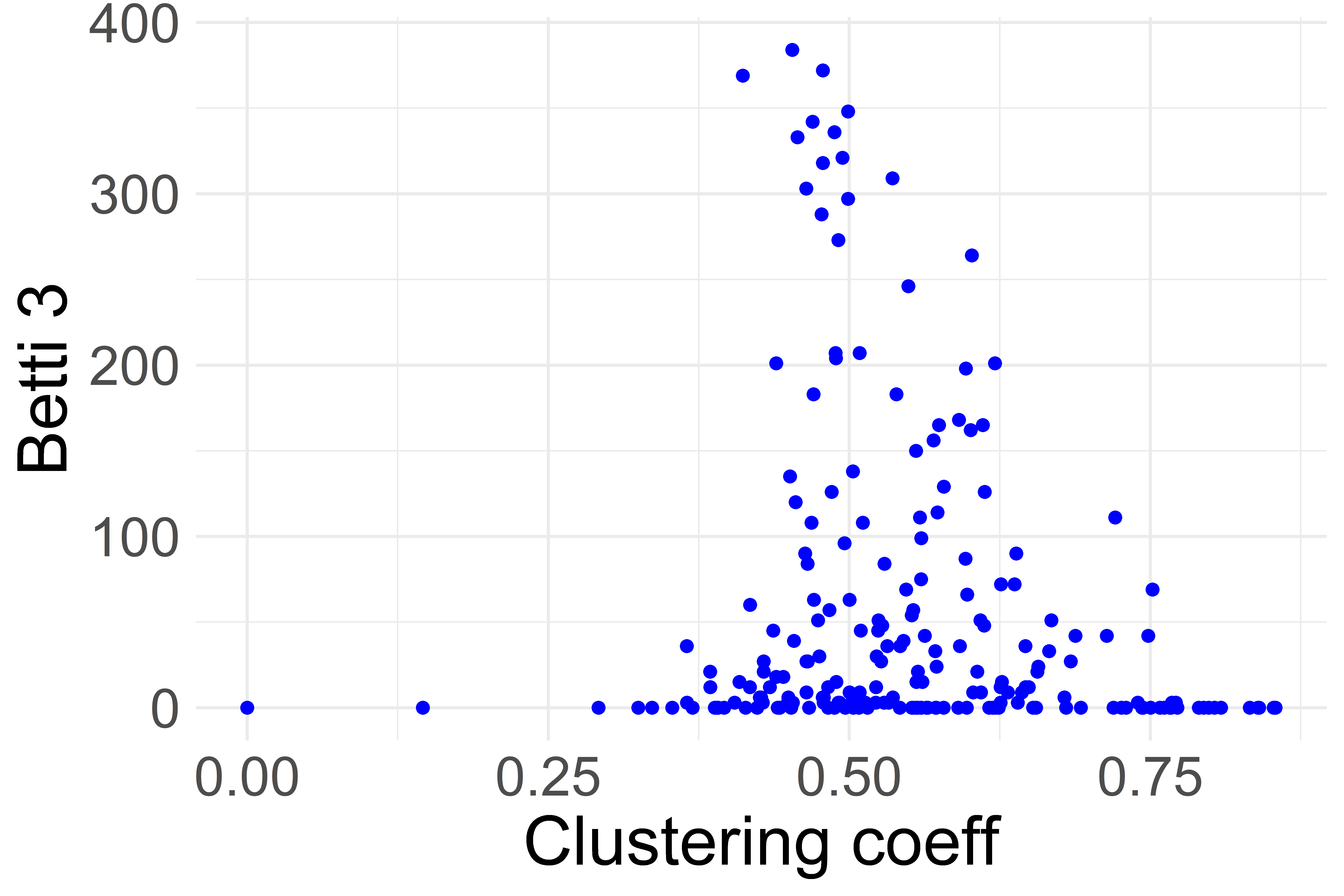} 
	\end{subfigure}
	\begin{subfigure}{.24\textwidth}
		\centering
		\includegraphics[width=.99\linewidth]{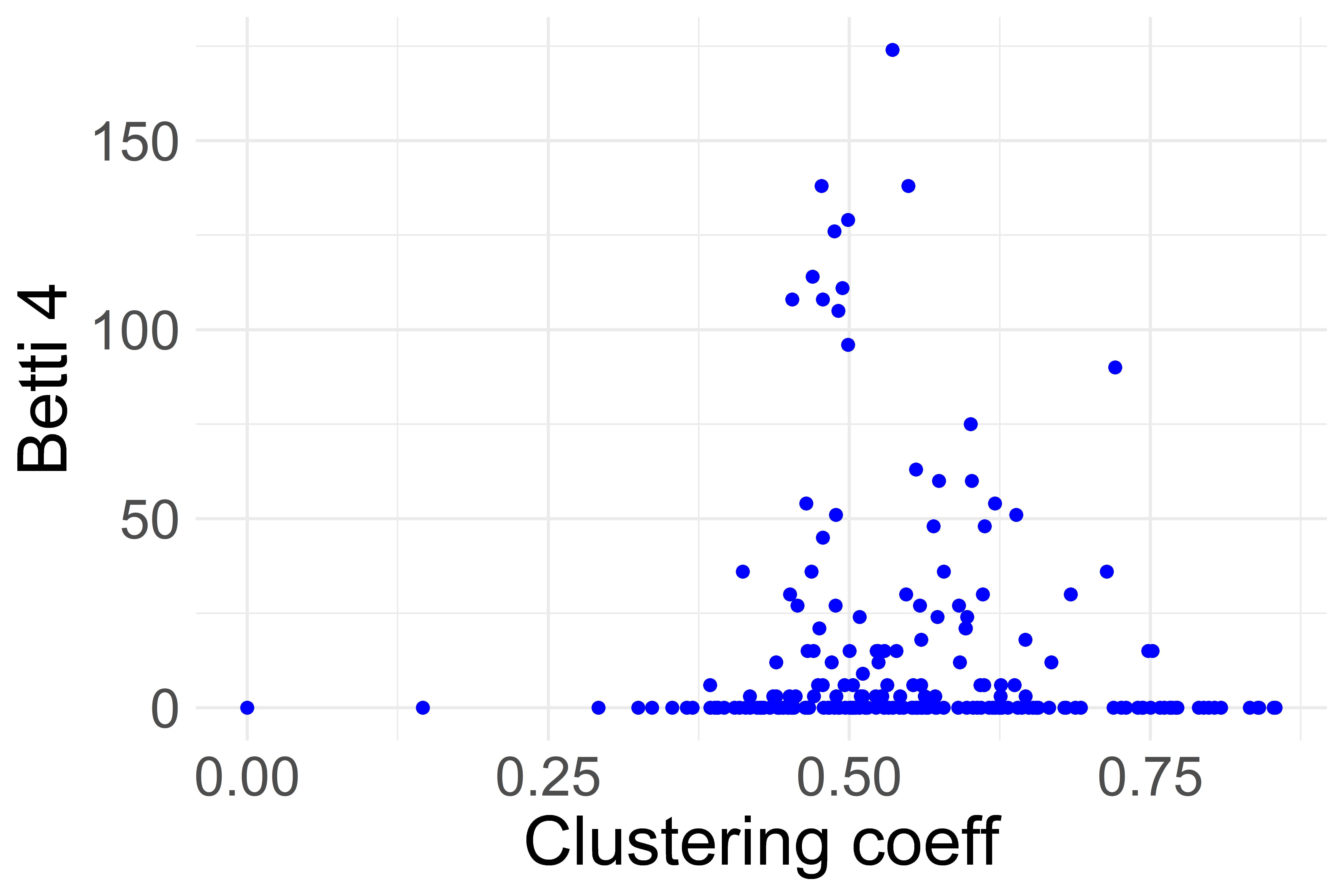}  \end{subfigure}
	\begin{subfigure}{.24\textwidth}
		\centering
		\includegraphics[width=.99\linewidth]{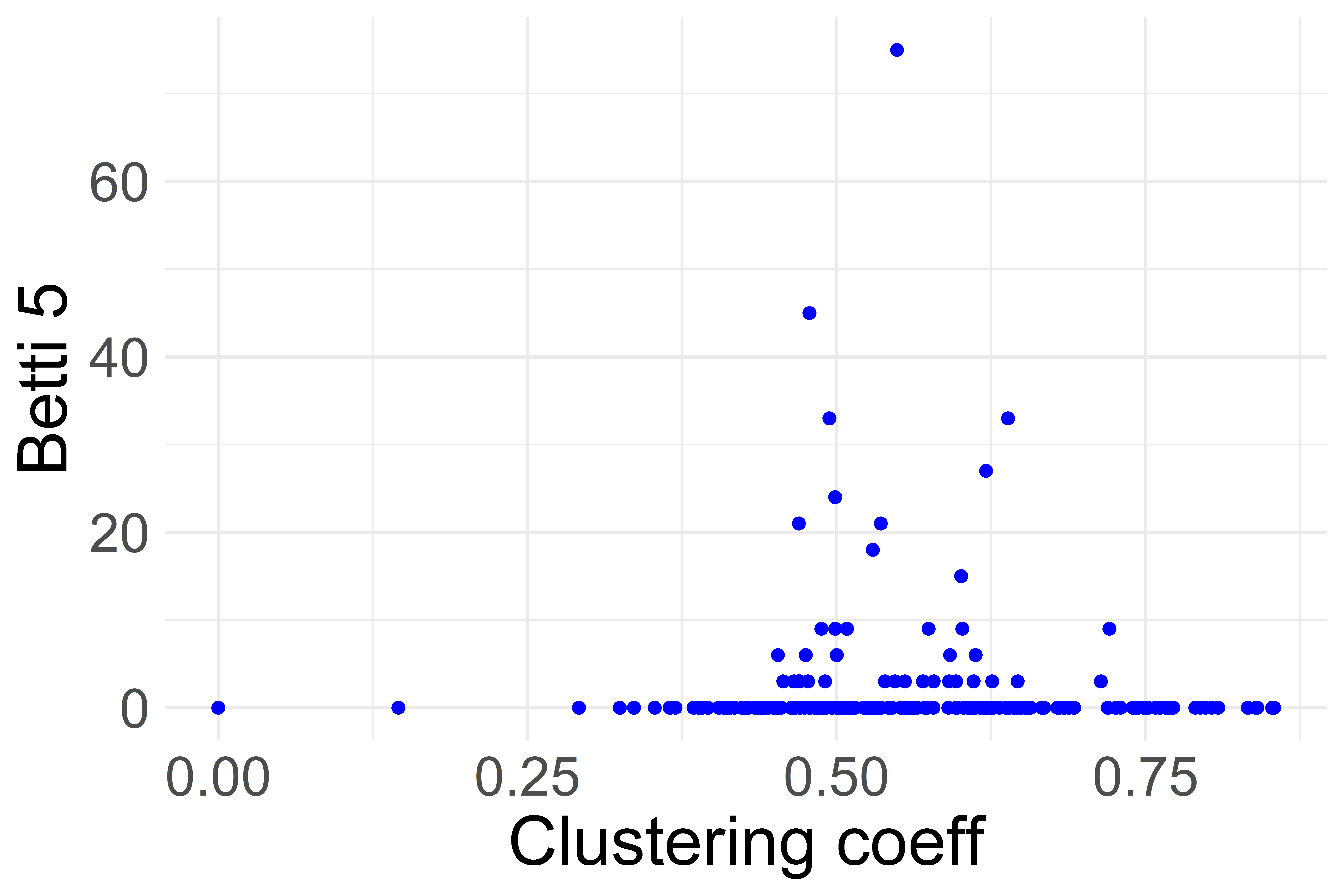} 
	\end{subfigure}
	\caption{\footnotesize Clustering coefficients vs. number of topological features in Facebook and Twitter datasets. Each data point is a graph instance. We observe hundreds of higher topological features in these datasets which can be highly useful for various graph learning tasks.}
	\label{fig:clussocial}
\end{figure*}

\begin{remark}  \label{remark:kahle} \normalfont [Higher PDs in Random Networks vs. Real-Life Networks] Note that by Kahle's seminal result~\cite{kahle2009topology}, to observe nontrivial Betti numbers for higher dimensions in Erd\'os-R\'enyi graphs $G(n,p)$, the average degree must be very high. In particular, for a graph $G(n,p)$, in order to have nontrivial $k^{th}$-homology in its clique complex, Kahle proved that for $p=n^\alpha$, $\alpha$ should be between $-1/k$ and $-1/(k+1)$. In terms of average degree $n\times p$, this means the average degree should be between $n^{(k-1)/k}$ and $n^{k/(k+1)}$. For instance, for dimension $k=2$, the average degree should be between $\sqrt{n}$ and $\sqrt[3]{n^2}$. For a graph order of $n=1000$, this implies that the average degree should be between $31$ and $100$ to have a nontrivial second homology in random networks. However, in real-life networks, our results show that higher Betti numbers are prevalent in much sparser graphs (see   \cref{fig:clussocial} and appendix \cref{fig:cluskernel}). These findings can be further used to derive error bounds and the associated loss of topological information when $G(n,p)$ is employed to approximate real-world network phenomena, for instance, in the case of synthetic power grid networks and other cyber-physical systems.  
\end{remark}

In the following, we give another effective method to reduce the size of a graph $\CG$ without affecting the persistence diagrams $PD_r(\CG)$ \textit{for any dimension} $r\geq 0$.

\section{PrunIT Algorithm}
\label{sec:trim}   

This section introduces another effective reduction technique for computing persistence diagrams of graphs induced by a filtering function. In particular, we show that for a graph $\CG$, and filtering function $f:\V\to \R$, removing (pruning) specific vertices from the graph does not change the persistent homology at any level. The result is valuable because the algorithm may reduce the vertex set considerably (Table \ref{tab:prunitresults}). Furthermore, as our experiments show, the reduced vertex set can significantly lower the simplex count, leading to much shorter computational times for persistent homology (see Figure \ref{fig:vertex} and appendix Figure \ref{fig:complex}).

In algebraic topology, homotopy is a very effective tool to compute topological invariants like homology, and fundamental group \cite{hatcher2002algebraic}. These topological invariants are homotopy invariant, meaning that if two spaces are homotopy equivalent, then their corresponding topological invariants are the same, e.g., $X\sim Y \Rightarrow H_i(X)=H_i(Y)$. We give a very natural homotopy construction to simplify a graph in the following.

For a given filtering function $f:\V\to \R$, let $\wh{\CG}_i$ be the clique complex of $\CG_i$ which induces the sublevel filtration $\wh{\CG}_0\subset \wh{\CG}_1\subset \wh{\CG}_2\subset ...\subset \wh{\CG}_m$. Let $PD_k(\CG,f)$ represent the $k^{th}$ persistence diagram for the sublevel filtration $\{\wh{\CG}_i\}$ as described above.

Now, we define \textit{dominated vertices} in $\CG$. Define the neighborhood of $u_0$ as $N(u_0)=\{u_0\}\cup\{v\in \V\mid e_{u_0v}\in \E\}$. In particular, $N(u_0)\subset \V$ is the set of all vertices adjacent to $u_0$, and $u_0$ itself.

\begin{definition} A vertex $u$ is \textit{dominated by} the vertex $v$ in $\CG$ if $N(u)\subset N(v)$. If there is such a vertex $v$, we call $u$ a \textit{dominated} vertex of $\CG$ (see Figure~\ref{fig:toydomination}). 
\end{definition}

\begin{wrapfigure}{r}{0.28\textwidth} 
	\vspace{-.2cm}
	\begin{center}
		\includegraphics[width=0.24\textwidth]{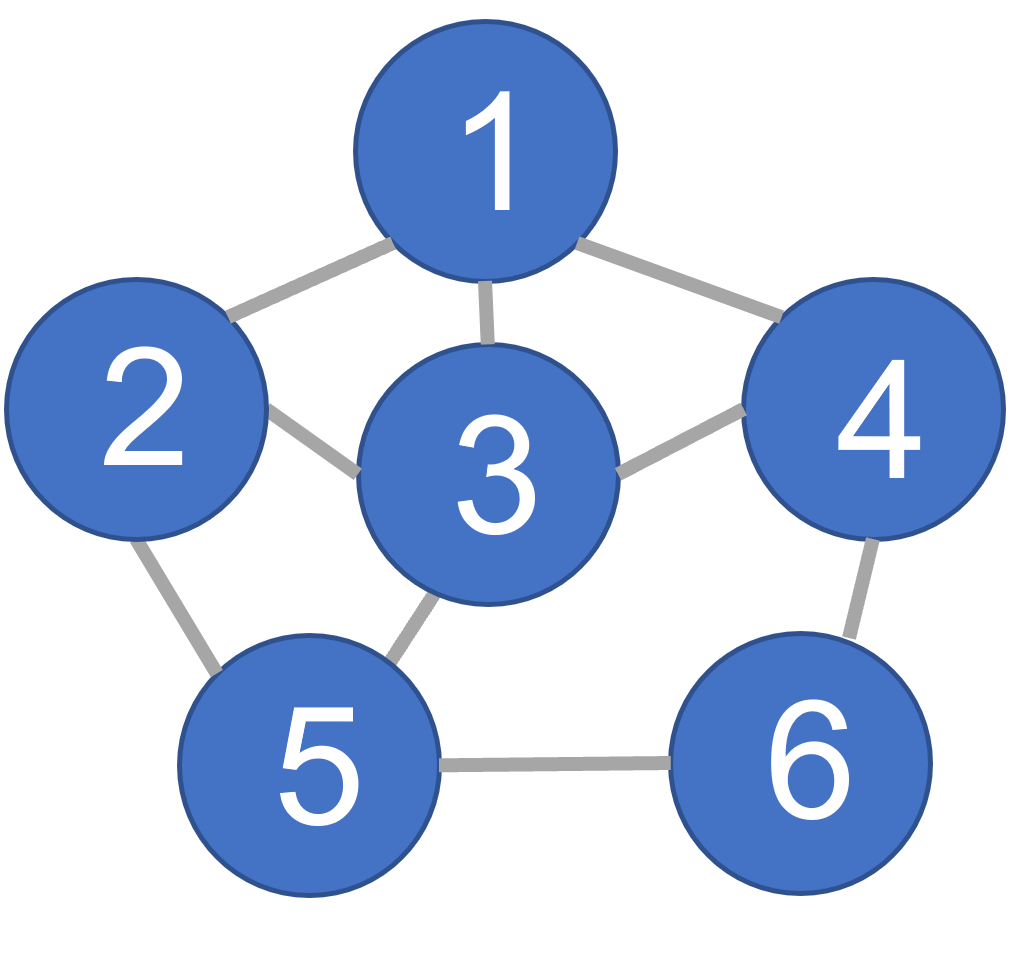}
		\caption{\footnotesize Vertex 3 dominates vertices 1 and 2 because all neighbors of 1 or 3 are neighbors of 3. There are no other dominated vertices.}
		\label{fig:toydomination}
	\end{center}
	\vspace{-.2cm}
\end{wrapfigure}

Removing a vertex $u$ from a graph $\CG$ creates the natural subgraph of $\CG$ obtained by removing the vertex $u$ and all adjacent edges from $\CG$, i.e. $\CG-\{u\}=\CG'=(\V',\E')$ where $\V'=\V-\{u\}$, and $\E'=\E-\{e_{uw}\in \E\}$ for any $w$. 

We can alternatively express these via the star notion. The \textit{star} $\mathbf{St}(u)$ of a vertex $u$ is the union of all simplices which contains $u$. Then, $u$ is dominated by $v$ if $\mathbf{St}(u)\subset \mathbf{St}(v)$. Similarly, removing a vertex $u$ from $\CG$ corresponds to removing $\mathbf{St}(u)$ from the clique complex $\wh{\CG}$, i.e. $\wh{\CG}-\mathbf{St}(u)=\wh{\CG}'$. A useful result is that removing a dominated vertex does not affect the homotopy type of the corresponding clique complexes.

\vspace{.2cm}
\begin{lemma} \label{lem:folding} Let $u$ be a dominated vertex in $\CG$. Let $\CG'=\CG-\{u\}$. Then the clique complexes $\wh{\CG}$ and $\wh{\CG}'$ are homotopy equivalent, i.e. $\wh{\CG}\sim \wh{\CG}'$.
\end{lemma}

\begin{proof} Notice that $\CG'$ is a subgraph of $\CG$, and hence $\wh{\CG'}$ is a subcomplex in $\wh{\CG}$. 
	Let $u$ be dominated by $v$ in $\CG$. Then, we can write a deformation retract from $\wh{\CG}$ to $\wh{\CG}'$ by pushing the edge $e_{uv}$ starting from $u$ toward $v$. In other words, by using the simplicial coordinates, one can define a homotopy $F:\wh{\CG}\times I\to \wh{\CG}$ which is identity on $\wh{\CG}'$ and pushing all the faces in $\wh{\CG}-\wh{\CG}'$ to the corresponding faces in $\wh{\CG}'$. This gives a homotopy equivalence $\wh{\CG}\sim \wh{\CG}'$. To visualize, in Figure~\ref{fig:toydomination}, one can push vertex $1$ in the clique complex $\wh{\CG}$ towards vertex $3$ along the edge between them. After the push, the $2$-simplices [1,2,3] and [1,3,4] are pushed to the edges [2,3] and [3,4] respectively.  See \cite{adamaszek2013clique, boissonnat2018computing} and \cite[Lemma 2.2]{boulet2010simplicial} for details.
\end{proof}

\begin{remark} \label{remark:collapsing} \normalfont [Collapsing] Note that this collapsing operation is adaptation of a well-known notion called \textit{deformation retract} in algebraic topology in a simplicial complex setting \cite{hatcher2002algebraic}. This operation keeps the homotopy type the same, and hence the homology does not change with this reduction. In \cite{adamaszek2013clique, boissonnat2018computing, boulet2010simplicial}, this is called \textit{folding} ($\CG$ folds onto $\CG-\{u\}$) or a \textit{strong collapse}. In these papers, the algorithm reduces simplicial complexes in the filtration one by one so that its associated clique complex keeps the same homotopy type. Our contribution here is to adapt this operation to the graph filtrations and define a smaller subgraph before the filtration step so that the induced simplicial complexes are homotopy equivalent. Since we prune the graph at the beginning of the process, our algorithm significantly reduces the computational costs for the induced persistence diagrams.
\end{remark}

In the following, we introduce the \textit{PrunIT Algorithm} by showing that removing a dominated vertex does not change the persistence diagrams of the graph. We give the proof in Appendix \ref{sec:proofs}.

\begin{theorem} \label{thm:reduction} Let $\CG=(\V,\E)$ be an unweighted graph, and $f:\V\to\R$ be a filtering function. Let $u\in\V$ be dominated by $v\in \V$ and $f(u)\geq f(v)$. Then, removing $u$ from $\CG$ does not change the persistence diagrams for sublevel filtration, i.e. for any $k\geq 0$ $$PD_k(\CG,f)=PD_k(\CG-\{u\},f).$$
\end{theorem}

\noindent {\em Outline of the proof:} The main idea is to employ the collapsing idea in the simplicial complexes of the filtration $\wh{\CG}_0\subset \wh{\CG}_1\subset \wh{\CG}_2\subset \dots\subset \wh{\CG}_m$  in a suitable way. In particular, the lemma above shows that if a vertex $u$ is dominated by a vertex $v$ in $\wh{\CG}_i$, then removing $\mathbf{St}(u)$ from $\wh{\CG}_i$ does not change the homotopy type. Hence, if we ensure that when $u$ first appears in the filtration $\{\wh{\CG}_i\}$, the dominated vertex $v$ is already there, then $u$ can be removed from all the simplicial complexes in the filtration; removing $u$ from the original graph before building the simplicial complexes does not affect the homotopy type of complexes in the filtration. The condition $f(u)\geq f(v)$ makes sure that whenever $u$ exists in $\{\wh{\CG}_i\}$, the dominant vertex is already there, and $u$ can be removed from all simplicial complexes, and hence from the graph $\CG$. We give the proof of the theorem in Appendix \ref{sec:proofs}.

Notice that the primary condition to remove dominated vertices from the graph ensures that the dominated vertex enters the filtration after its dominating counterpart. With the PrunIT Algorithm, we show that removing the dominated vertex does not change the homotopy type of the simplicial complexes in the filtration. As homotopy equivalence implies the equivalences of homology groups at all levels, the reduction with this algorithm works in all dimensions. Furthermore, while coral reduction works above the corresponding dimension $(j>k)$, the PrunIT algorithm works in any dimension.

\begin{remark} \label{remark:superlevel} \normalfont [Superlevel Filtration] \normalfont The same proof applies to the superlevel filtration by changing the condition $f(u)\geq f(v)$ to $f(u)\leq f(v)$ in the theorem. In particular, if $PD_k^\mathrm{v}(\CG,f)$ represents the $k^{th}$ PD for superlevel filtration, then with the condition $f(u)\leq f(v)$, we would have $PD_k^\mathrm{v}(\CG,f)=PD_k^\mathrm{v}(\CG-\{u\},f)$ for any $k\geq 0$.
	Notice that if one takes $f$ to be the degree function and uses the superlevel filtration, then the theorem automatically holds for any dominated vertex as $deg(u)\leq deg(v)$ when $u$ is dominated by $v$.  
\end{remark}

\begin{remark} \label{remark:detecting_dominated} \normalfont [Detecting Dominating Vertices]
	The dominating vertices can be computed by using the following approach (Algorithm is given in appendix Section~\ref{sec:algorithms}). Let $\A=(a_{ij})$ be the adjacency matrix for a graph $\CG$. Given $v_{i_0}\in\V$, consider all $j$'s with $a_{i_0j}=1$. Check if $v_{i_0}$ is dominated by $v_j$ by comparing the rows $R_{i_0}$ and $R_j$, i.e. for any $k\neq j$ with $a_{i_0k}=1$, check whether $a_{j_k}=1$. If this holds, $v_j$ dominates $v_{i_0}$. Removing $i_0^{th}$ row $R_{i_0}$ and $i_0^{th}$ column $C_{i_0}$ from $\A$ corresponds to removing $v_{i_0}$ from $\CG$. Essentially, vertex $v_{i_0}$ is compared to each neighbor $v_j$ by checking whether $v_{i_0}$ is already a neighbor of each of $v_j$'s neighbors. These checks require iterating over each vertex, searching vertex neighbors in the graph and getting the neighbors of each neighbor. The computational complexity is therefore $\mathcal{O}(|\V|\times d^2)$ where $d$ is the average degree in the graph. 
\end{remark}

While our main focus is the most common method, sublevel/superlevel filtration, in the application of PH in graph setting, our PrunIt algorithm works perfectly well with another common method, power filtration~\cite{aktas2019persistence}, as well, i.e. removing a dominated vertex does not change persistence diagrams.
\begin{theorem} \label{thm:Prunit_power} [PrunIt for Power Filtration] Let $\CG=(\V,\E)$ be an unweighted connected graph. Let $\wh{PD}_k(\CG)$ represent $k^{th}$ persistence diagram of $\CG$ with power filtration. Let $u\in\V$ be dominated by any other vertex in $\V$. Then, for any $k\geq 1$, $$\wh{PD}_k(\CG)=\wh{PD}_k(\CG-\{u\}).$$
\end{theorem}

The proof of this result is given in appendix Section~\ref{sec:proof_Prunit}.

\subsection*{Combining the CoralTDA and PrunIT Algorithms}

Even though both algorithms are quite effective by themselves, we significantly reduce the computational costs (\Cref{fig:combinedresults}) by combining them as follows. For a given graph $\CG=(\V,\E)$ and filtering function $f:\V\to \R$, one can start by trimming all dominated vertices with respect to $f$, and get a smaller graph $\CG'$. We have already proven that $PD_k(\CG)=PD_k(\CG')$. Then, one can take the $k$-core of this smaller graph $\CG'$ to compute higher persistence diagrams of the original graph $\CG$ as before. In particular, by applying both reduction algorithms, for any $k\geq 0$, we obtain $$PD_k(\CG)=PD_k(\CG')=PD_k((\CG')^{k+1}).$$

\section{Experiments and Discussion}
\label{sec:experiments} 

We apply our new approaches to three types of datasets. The details of datasets are provided in Table~\ref{tab:prunitresults} and appendix Table~\ref{tab:dataset}. 

\noindent {\em Graph classification datasets} consists of biological kernel~\cite{KKMMN2016} and ego networks from \texttt{TWITTER} and \texttt{FACEBOOK}~\cite{mcauley2012learning}. 
\noindent {\em Node classification datasets} includes \texttt{CITESEER} and \texttt{CORA}~\cite{nr} and Open Graph Benchmark citation (paper cites paper) \texttt{OGB-ARXIV} and \texttt{OGB-MAG}~\cite{hu2020open} networks. \noindent {\em Large networks dataset} contains 11 large networks of 100K-1M vertices from the Stanford  Repository~\cite{snapnets}. 

We used an AMD Ryzen 5 2100 MHZ 4 core computer in our R, Python and Java experiments.

We evaluate both algorithms by comparing vertex and edge sets and the total run time for the reduced graph with respect to the original graph. In the rest of this manuscript, we compute the vertex set reduction as $100\times( \left |\V\right |-\left |\V^\prime\right |) / \left |\V\right |$ where $\V^\prime$ is the vertex count in the reduced graph. Edge and time reductions are computed similarly.

\subsection{Reduction on Graph Classification Datasets}
In this task, our goal is to evaluate the reduction of computational costs when we use the CoralTDA and PrunIT algorithms on datasets chosen from different graph classification tasks. We used one of the most commonly used functions in these experiments, the degree function with sublevel filtration.

\begin{figure*}[!ht]
	\centering
	\begin{subfigure}{.24\textwidth}
		\centering 
		% include first image
		\includegraphics[width=.95\linewidth]{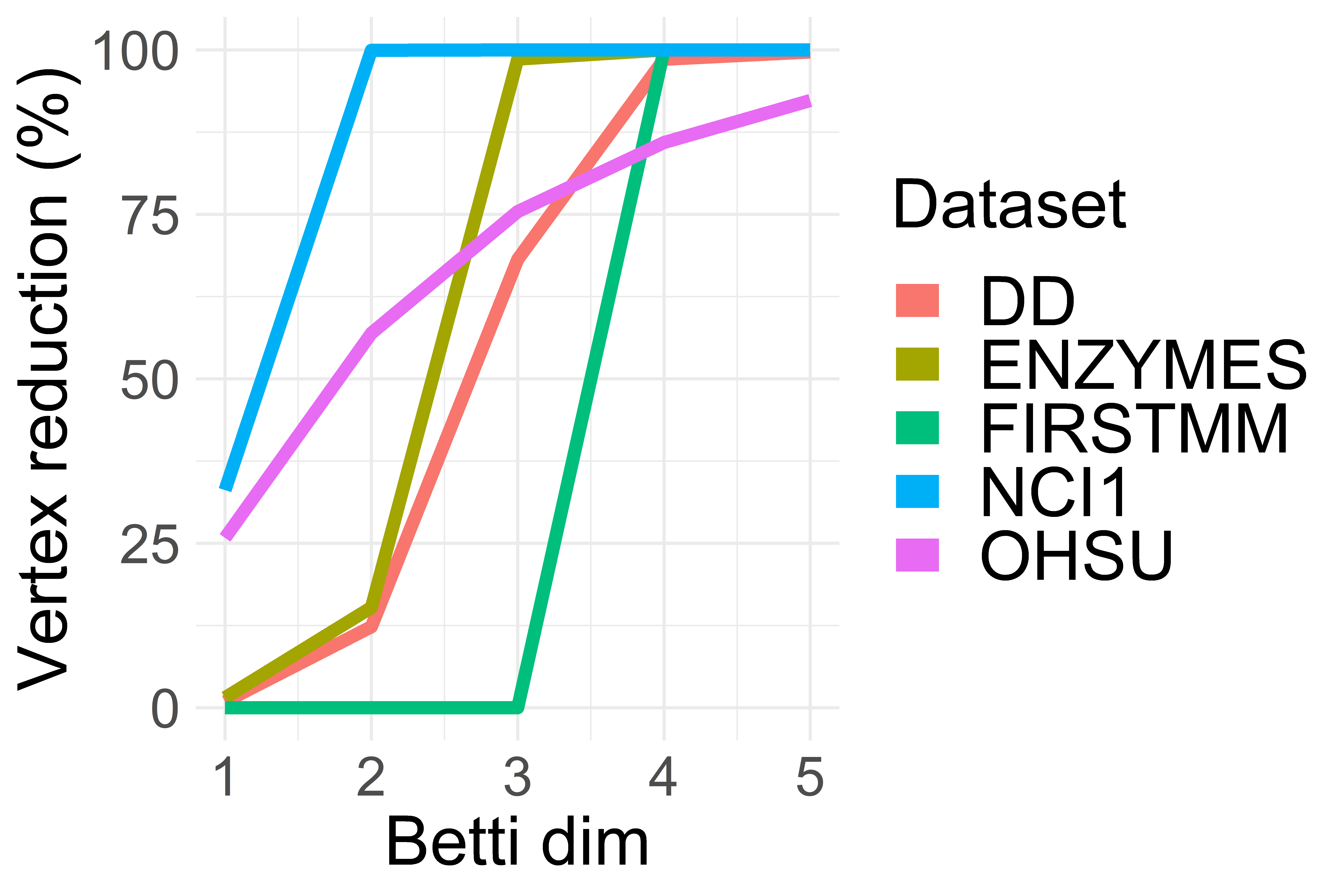}  
	\end{subfigure}
	\begin{subfigure}{.24\textwidth}
		\centering 
		% include first image
		\includegraphics[width=.95\linewidth]{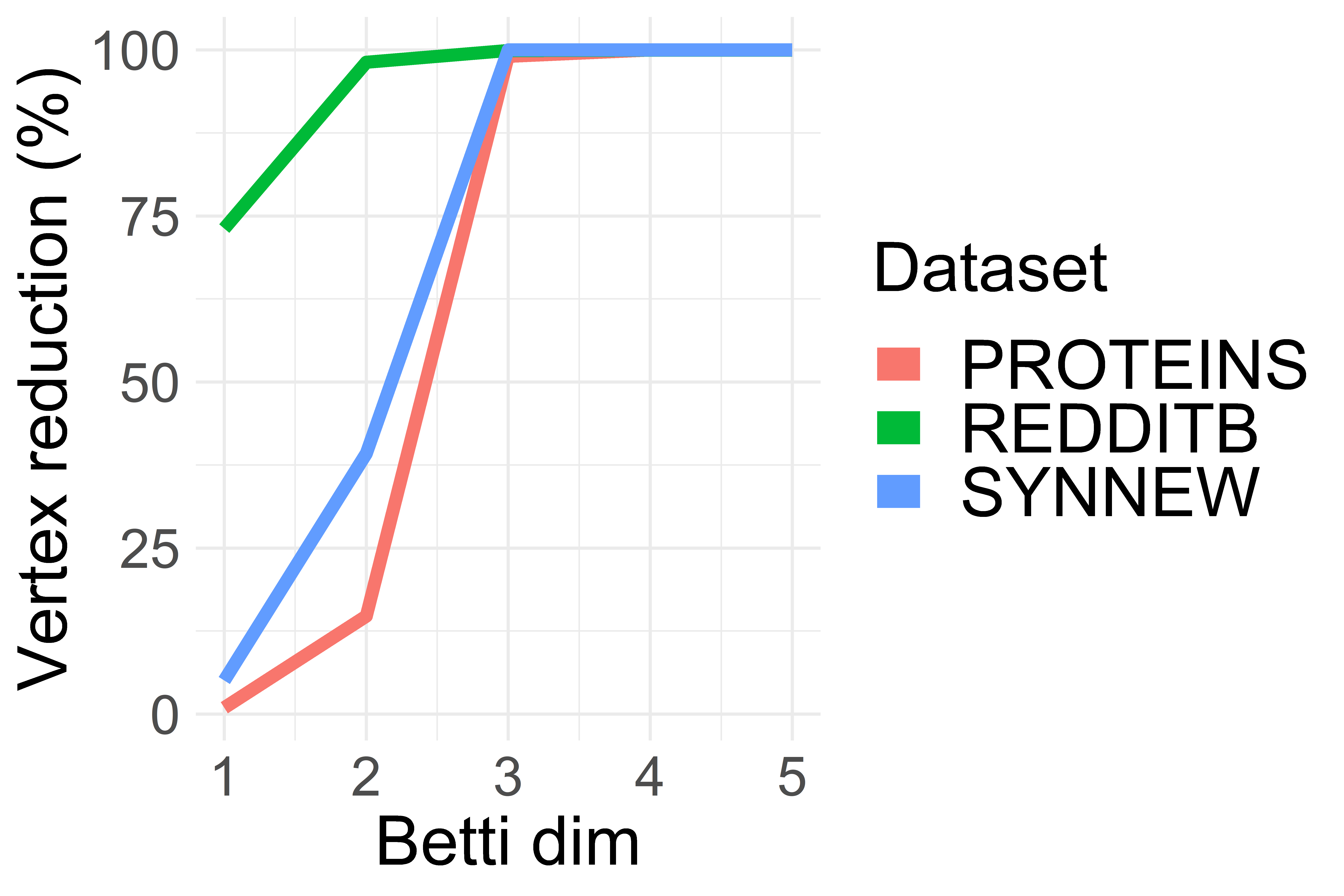}  
	\end{subfigure}
	\begin{subfigure}{.24\textwidth}
		\centering
		% include first image
		\includegraphics[width=.95\linewidth]{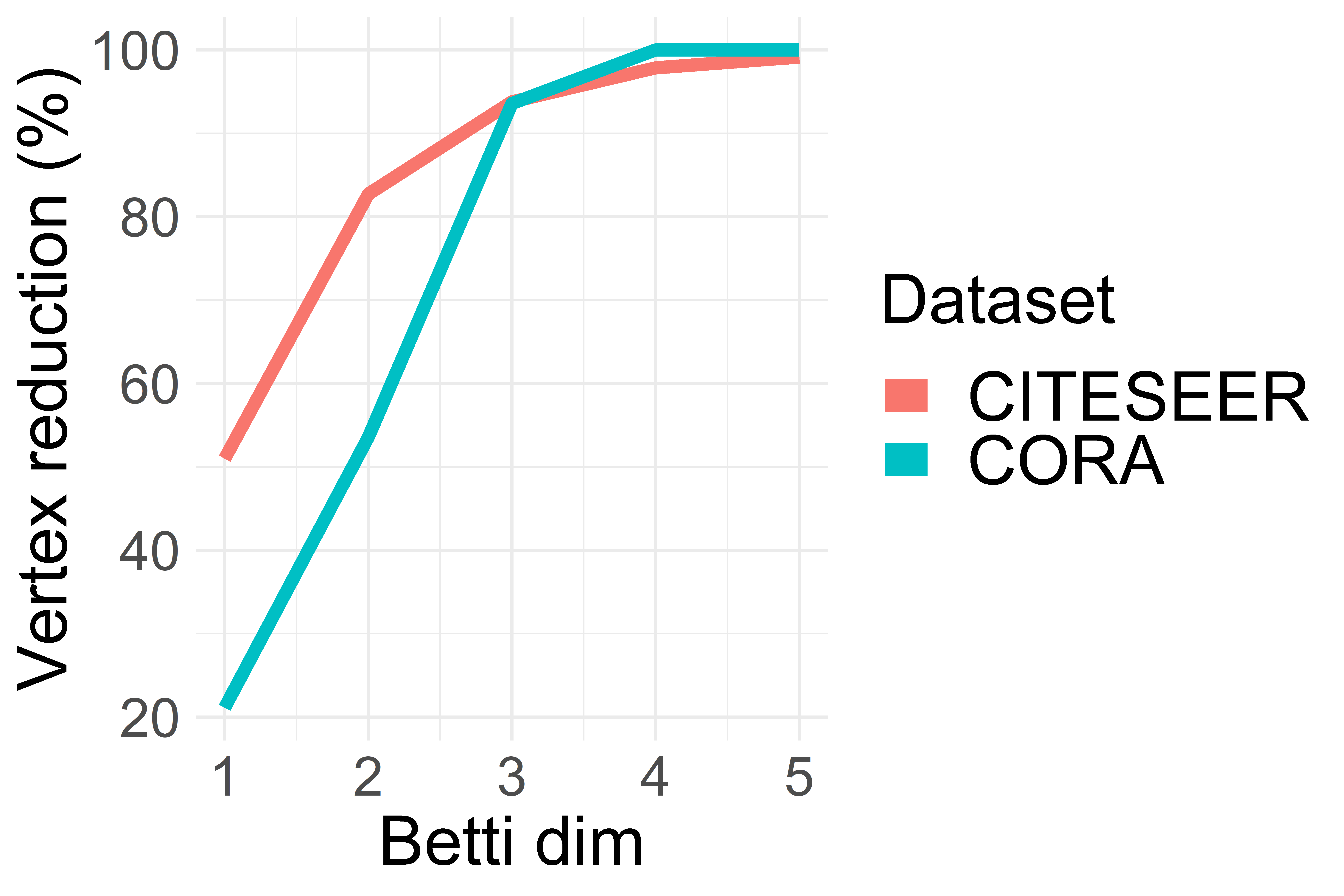}  
	\end{subfigure}
	\begin{subfigure}{.24\textwidth}
		\centering
		% include second image
		\includegraphics[width=.95\linewidth]{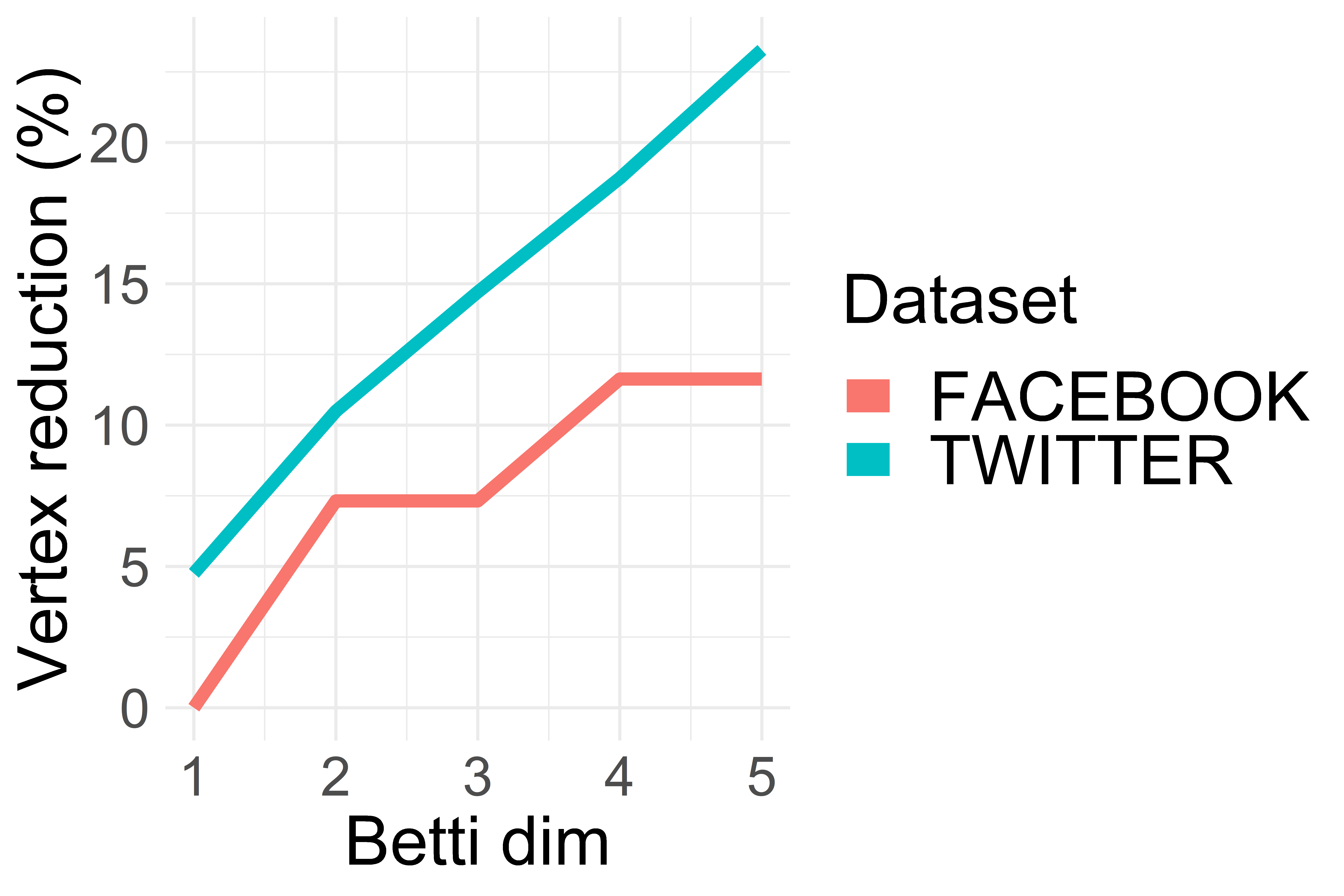}  
	\end{subfigure}
	\caption{\footnotesize CoralTDA vertex reduction in graph and node classification datasets (higher is better). Reduction values are averages from graph instances of the datasets (\texttt{CORA} and \texttt{CITESEER} node classification datasets contain a single graph instance only). \texttt{FACEBOOK} and \texttt{TWITTER} datasets are reduced by 20\% for $k>4$, whereas in other datasets graphs are reduced to empty sets.} 
	\label{fig:vertex}
\end{figure*}

In Figure~\ref{fig:vertex}, we show the vertex reduction when using CoralTDA for computations of $PD_k(\CG)$ for dimensions (Betti) $k=1$ to $k=5$. At dimension $k=4$ and $k=5$, CoralTDA reduces 10 datasets by 100\%, i.e., these  datasets have trivial $PD_k(\CG)$ for $k\geq 4$. Even at smaller dimensions, CoralTDA can reduce the vertex set by 25\%-75\%.

Figure~\ref{fig:reductiondomination} shows reduction percentages by the PrunIT algorithm. 
\texttt{FIRSTMM} and \texttt{SYNNEW} datasets are reduced by less than 10\%; however the other 11 datasets are reduced by at least 35\%.  The lower reduction on \texttt{FIRSTMM} and  \texttt{SYNNEW} are due to stronger cores on the networks. \texttt{SYNNEW} is synthetically created, but \texttt{FIRSTMM} is created from 3d point cloud data and categories of various household objects. We believe that the physical proximity of similar objects (e.g., chairs are close to each other) in a household creates a denser community structure in the \texttt{FIRSTMM} dataset, which in turn results in strong cores. 

We further report reductions in computational time (Figure~\ref{fig:time}), edge set (Figure~\ref{fig:edge}), and simplex count (Figure~\ref{fig:complex}) in the Appendix. 

\subsection{Reduction on Node Classification Datasets}
In this task, our goal is to compute the reduction of computational costs by using CoralTDA and PrunIT algorithms on datasets chosen from node classification tasks. The CoralTDA results are computed over \texttt{CITESEER} and \texttt{CORA} networks and shown in \Cref{fig:vertex} with more than 20\% reduction for the first and higher dimensional persistence.

\begin{figure*}[t]
	\centering
	\begin{subfigure}{.49\textwidth}
		\centering 
		% include first image
		\includegraphics[height=.65\linewidth]{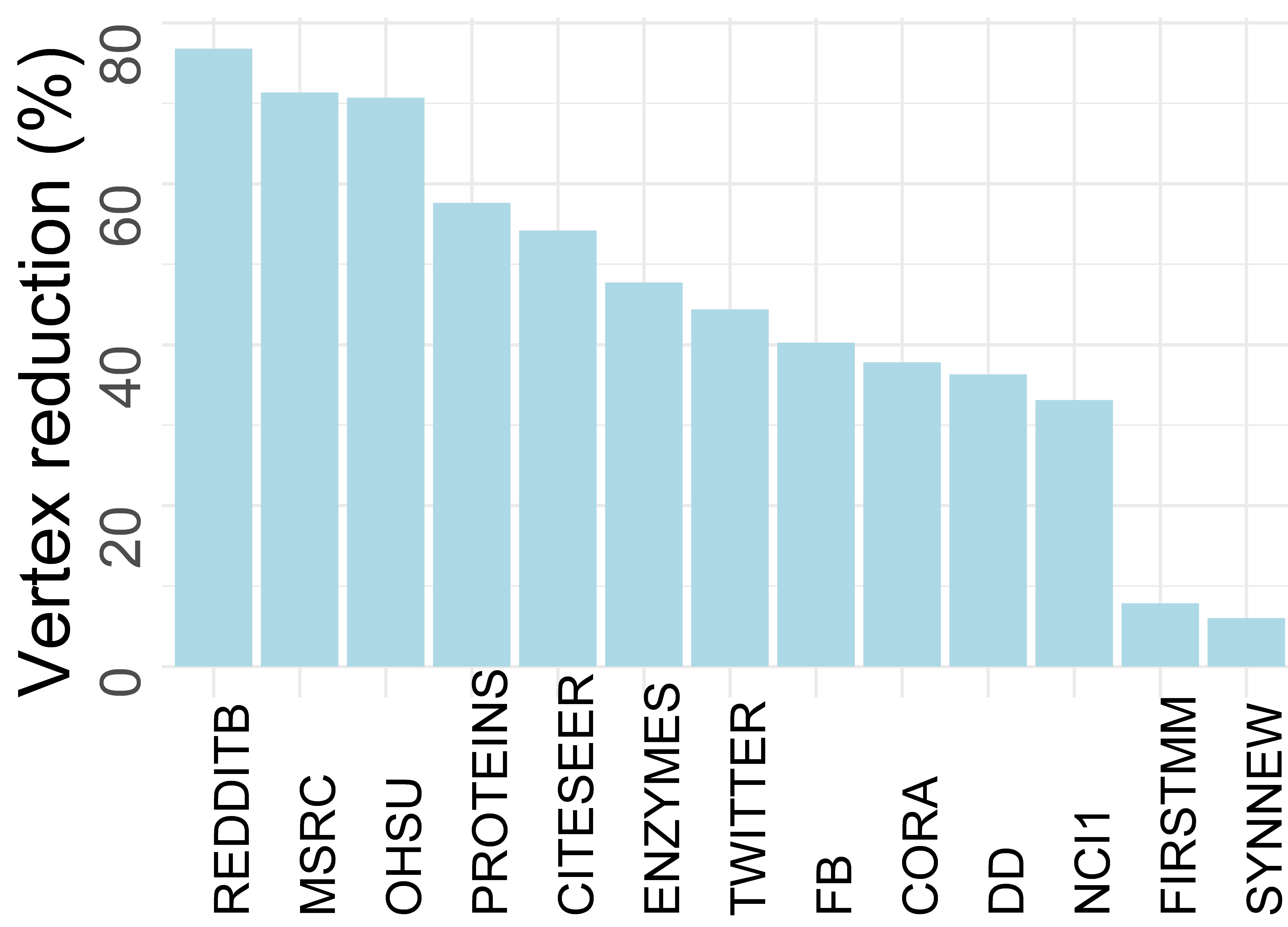} 
		\caption{\footnotesize Vertex reduction by PrunIT algorithm in the superlevel filtration. Results are averages of graph instances from the datasets. }
		\label{fig:reductiondomination}
	\end{subfigure}
	~
	\begin{subfigure}{.49\textwidth}
		\centering 
		% include first image
		\includegraphics[height=.65\linewidth]{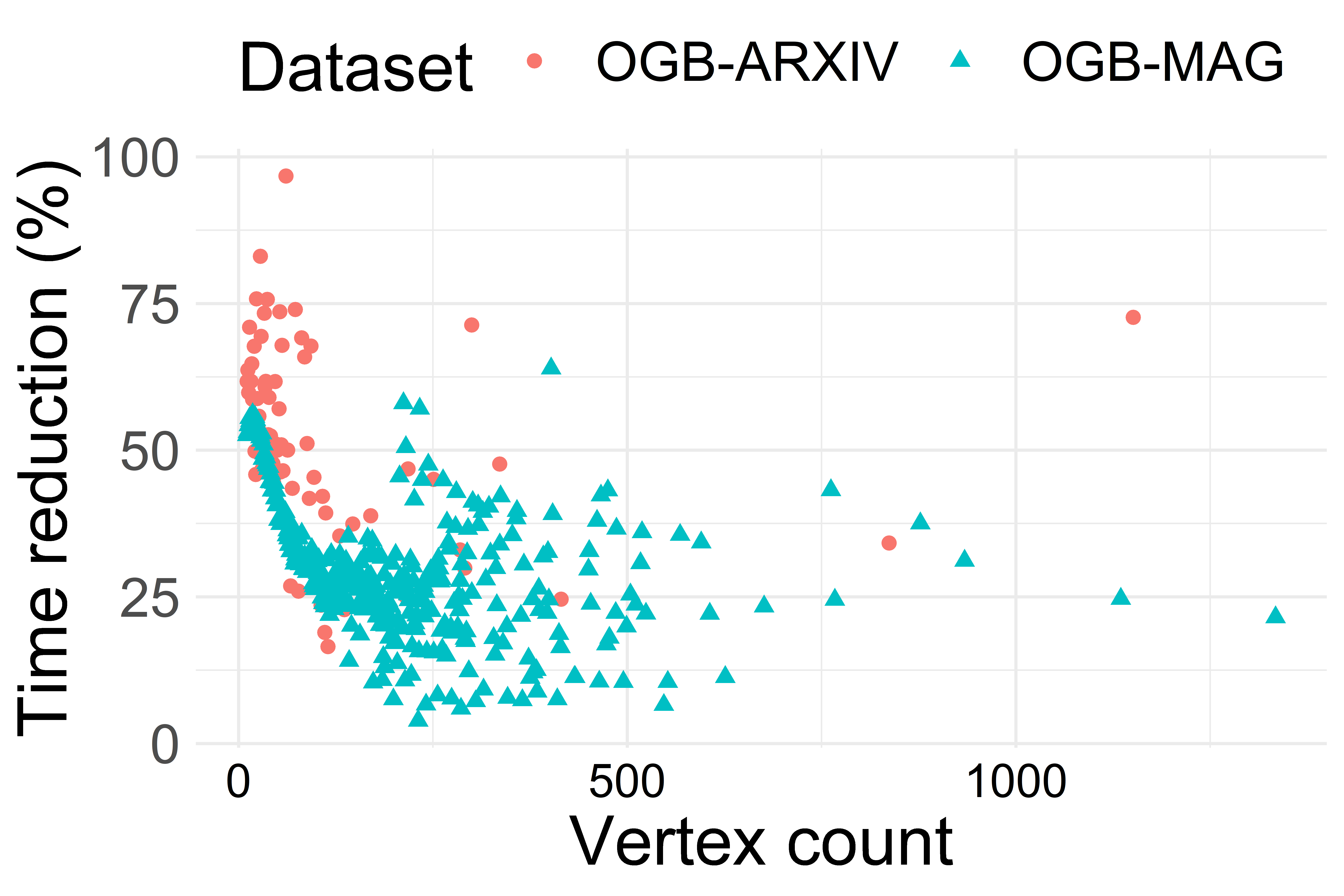} 
		\caption{\footnotesize PrunIT reduction in OGB node classification dataset. Each data point is an ego network. Even for large networks, time reduction rates can reach 75\%.}
		\label{fig:dominatedOGB}
	\end{subfigure}
	\caption{\footnotesize PrunIt vertex and time reduction in graph datasets.} 
	\label{fig:new1}
\end{figure*}

In node classification, we can also analyze the k-hop ($k\ge 1$) neighborhood of a vertex with topological features (such as Betti-0) and use the computed persistence diagram to classify the vertex. Such an approach has yielded SOTA results with significant improvement in accuracy by using 0-dimensional persistence~\cite{chen2021topological}. However, the computational costs of persistent homology are non-negligible in large graphs, even for 0-dimensional features. For example, in Open Graph Benchmark datasets~\cite{hu2020open}, one must compute persistence diagrams for each vertex in 100k to 111M vertex graphs.

We apply the PrunIT algorithm to two graphs, Arxiv and MAG, from the Open Graph Benchmark to compute time reduction in persistence diagram computations. We follow the approach in~\cite{chen2021topological} and extract the 1-hop neighborhood of each ego vertex.  
We use the degree function as filtering function as before. In \Cref{fig:dominatedOGB}, we show the reduction in computational time for 0-dimensional persistence. We compute the time costs of PrunIT by considering all the algorithm steps: finding and removing the dominated vertices, creating an induced graph with the vertices, and running 0-dimensional persistent homology on the graph by using vertex degrees as the filtering function. As \Cref{fig:dominatedOGB} shows, we see more than 25\% reduction in computation time in most graphs. Specifically, on average,  computation times of 0-dimensional persistence on \texttt{OGB-ARXIV} networks are reduced by 37\%, and those of \texttt{OGB-MAG} networks are reduced by 23\%. The results show that we can mitigate the computational costs of persistence homology by using the PrunIT algorithms.

\begin{table}[h]
	\centering
	\scriptsize
	\caption{\footnotesize PrunIt reductions in the number of vertices and edges. }
	\label{tab:prunitresults}
	\begin{tabular}{l r c r c}
		\toprule
		Dataset&$\|V\|$&$\|V\|$ Reduction $(\uparrow)$&$\|E\|$ & $\|E\|$ Reduction $(\uparrow)$\\
		\midrule
		com-youtube&1134890&59\%&2987624&25\%\\
		com-amazon&334863&37\%&925872&40\%\\
		com-dblp&317080&72\%&1049866&65\%\\
		web-Stanford&281903&67\%&1992636&76\%\\
		emailEuAll&265214&95\%&364481&94\%\\
		soc-Epinions1&75879&57\%&405740&14\%\\
		p2pGnutella31&62586&46\%&147892&20\%\\
		Brightkite\_edges&58228&48\%&214078&21\%\\
		Email-Enron&36692&76\%&183831&38\%\\
		CA-CondMat&23133&69\%&93439&65\%\\
		oregon1\_010526&11174&62\%&23409&48\%\\
		\bottomrule
	\end{tabular}
\end{table}

\subsection{Reduction on Large Networks}
\label{sec:combined}

Our goal is to combine PrunIt and CoralTDA algorithms to achieve the maximum vertex and edge reduction in large networks in these experiments.

Table~\ref{tab:prunitresults} shows that on the biggest network of \texttt{com-youtube}, we eliminate 59\% of the vertices when we only apply the PrunIt  (on average 62\% in all datasets). The reduction is as high as 95\% (in \texttt{emailEuAll}). Similarly, PrunIt creates significant edge reduction; 40\% of all edges are removed on average. \Cref{fig:combinedresults} shows the reduction when we apply both CoralTDA and PrunIt on large networks. Even for low cores of 2 and 3, the combined algorithms reach a vertex reduction rate of 78\%. These results show that our algorithms can effectively reduce large networks to more manageable sizes.

\begin{figure}[h]
	\centering
	\includegraphics[width=.7\linewidth]{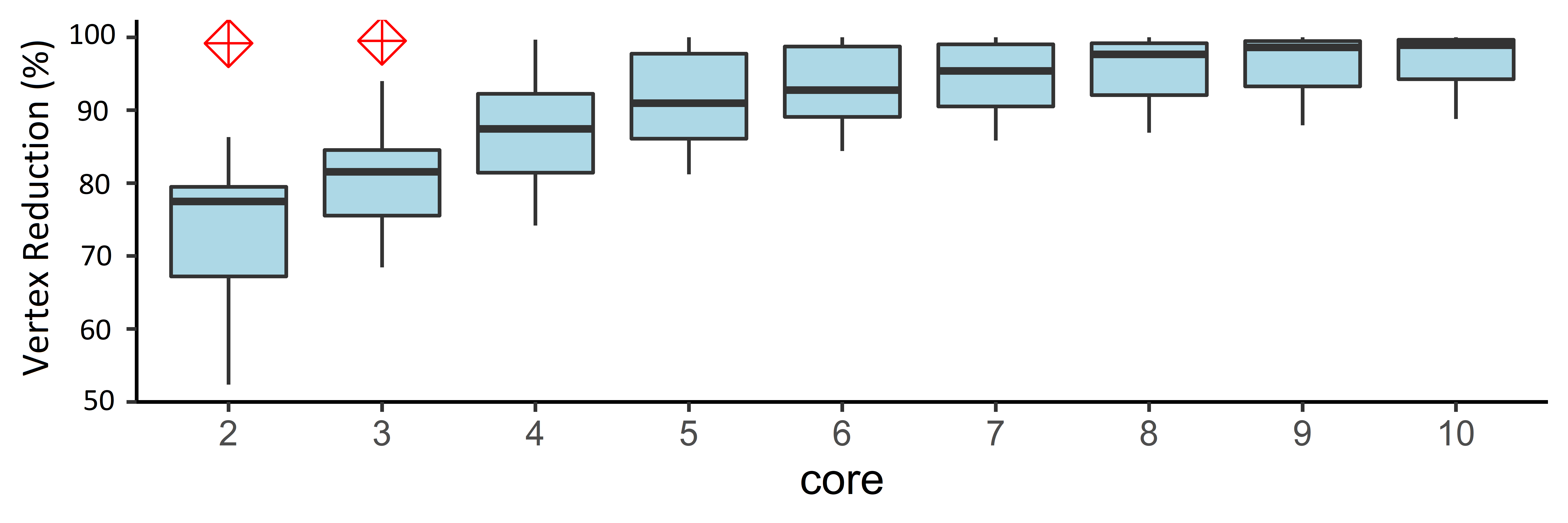}
	\caption{\footnotesize Vertex reduction results for 11 large datasets after the application of PrunIt and CoralTDA algorithms. \texttt{emailEuAll} is the outlier for the 2nd and 3rd cores (shown with a crossed square).}
	\label{fig:combinedresults}
\end{figure}

\section{Conclusion} 
\label{sec:conclusion}
We have proposed two new highly effective algorithms to significantly reduce the computational costs of TDA methods on graphs. While coral reduction is very effective for higher persistence diagrams, PrunIt is highly efficient, in general. Our experiments have showed that even for lower dimensional topological features, such as $k=1$, for some datasets our methods can reduce graph order by up to 95\%, which alleviates  computational costs substantially. Furthermore, in most graph datasets we reduce graph sizes by 100\% for 3rd or higher dimensions. Our methods provides a novel solution for efficient application of the powerful TDA methods on large networks and build a bridge between the graph theory and TDA,
opening a pathway for broader applicability of topological graph learning in practice. 
  
 \section{Acknowledgments}

This material is based upon work sponsored by the Canadian NSERC Discovery Grant {RGPIN-2020-05665}, NSF of USA under award number ECCS 2039701,  
OAC-1828467,  DMS-1925346, CNS-2029661, OAC-2115094, ARO award W911NF-17-1-0356, and Simons Collaboration Grant \# 579977. 
Part of this material is also based upon work supported by (while serving at) the National Science Foundation.
Any opinions, findings, and
conclusions or recommendations expressed in this material are
those of the author(s) and do not necessarily reflect the views
of the National Science Foundation.

\bibliographystyle{plain}
\bibliography{peacecorps.bib}

\begin{thebibliography}{10}

\bibitem{adamaszek2013clique}
Micha{\l} Adamaszek.
\newblock Clique complexes and graph powers.
\newblock {\em Israel Journal of Mathematics}, 196(1):295--319, 2013.

\bibitem{akcora2019bitcoinheist}
Cuneyt~Gurcan Akcora, Yitao Li, Yulia~R. Gel, and Murat Kantarcioglu.
\newblock Bitcoinheist: Topological data analysis for ransomware prediction on
  the bitcoin blockchain.
\newblock In Christian Bessiere, editor, {\em Proceedings of the Twenty-Ninth
  International Joint Conference on Artificial Intelligence, {IJCAI} 2020},
  pages 4439--4445. Ijcai.org, 2020.

\bibitem{aktas2019persistence}
Mehmet~E Aktas, Esra Akbas, and Ahmed El~Fatmaoui.
\newblock Persistence homology of networks: methods and applications.
\newblock {\em Applied Network Science}, 4(1):1--28, 2019.

\bibitem{amezquita2020shape}
Erik~J Am{\'e}zquita, Michelle~Y Quigley, Tim Ophelders, Elizabeth Munch, and
  Daniel~H Chitwood.
\newblock The shape of things to come: Topological data analysis and biology,
  from molecules to organisms.
\newblock {\em Developmental Dynamics}, 249(7):816--833, 2020.

\bibitem{batagelj2003m}
Vladimir Batagelj and Matjaz Zaversnik.
\newblock An o(m) algorithm for cores decomposition of networks.
\newblock {\em CoRR}, cs.DS/0310049, 2003.

\bibitem{benson2018simplicial}
Austin~R Benson, Rediet Abebe, Michael~T Schaub, Ali Jadbabaie, and Jon
  Kleinberg.
\newblock Simplicial closure and higher-order link prediction.
\newblock {\em Proceedings of the National Academy of Sciences},
  115(48):E11221--E11230, 2018.

\bibitem{boissonnat2018computing}
Jean{-}Daniel Boissonnat and Siddharth Pritam.
\newblock Computing persistent homology of flag complexes via strong collapses.
\newblock In Gill Barequet and Yusu Wang, editors, {\em 35th International
  Symposium on Computational Geometry, SoCG 2019, June 18-21, 2019, Portland,
  Oregon, {USA}}, volume 129 of {\em LIPIcs}, pages 55:1--55:15. Schloss
  Dagstuhl - Leibniz-Zentrum f{\"{u}}r Informatik, 2019.

\bibitem{boissonnat2020edge}
Jean{-}Daniel Boissonnat and Siddharth Pritam.
\newblock Edge collapse and persistence of flag complexes.
\newblock In Sergio Cabello and Danny~Z. Chen, editors, {\em 36th International
  Symposium on Computational Geometry, SoCG 2020, June 23-26, 2020,
  Z{\"{u}}rich, Switzerland}, volume 164 of {\em LIPIcs}, pages 19:1--19:15.
  Schloss Dagstuhl - Leibniz-Zentrum f{\"{u}}r Informatik, 2020.

\bibitem{boissonnat2021strong}
Jean-Daniel Boissonnat, Siddharth Pritam, and Divyansh Pareek.
\newblock Strong collapse and persistent homology.
\newblock {\em Journal of Topology and Analysis}, pages 1--29, 2021.

\bibitem{boot2016algorithms}
Coen Boot.
\newblock Algorithms for determining the clustering coefficient in large
  graphs.
\newblock {B.S.} thesis, Utrecht University, 2016.

\bibitem{boulet2010simplicial}
Romain Boulet, Etienne Fieux, and Bertrand Jouve.
\newblock Simplicial simple-homotopy of flag complexes in terms of graphs.
\newblock {\em European Journal of Combinatorics}, 31(1):161--176, 2010.

\bibitem{bruillard2016anomaly}
Paul Bruillard, Kathleen Nowak, and Emilie Purvine.
\newblock Anomaly detection using persistent homology.
\newblock In {\em 2016 Cybersecurity Symposium (CYBERSEC)}, pages 7--12. IEEE,
  2016.

\bibitem{burleson2020k}
Kate Burleson-Lesser, Flaviano Morone, Maria~S Tomassone, and Hern{\'a}n~A
  Makse.
\newblock K-core robustness in ecological and financial networks.
\newblock {\em Scientific Reports}, 10(1):1--14, 2020.

\bibitem{cai2020understanding}
Chen Cai and Yusu Wang.
\newblock Understanding the power of persistence pairing via permutation test.
\newblock {\em CoRR}, abs/2001.06058, 2020.

\bibitem{carlsson2009topology}
Gunnar Carlsson.
\newblock Topology and data.
\newblock {\em Bulletin of the American Mathematical Society}, 46(2):255--308,
  2009.

\bibitem{carstens2013persistent}
Corrie~J Carstens and Kathy~J Horadam.
\newblock Persistent homology of collaboration networks.
\newblock {\em Mathematical Problems in Engineering}, 2013, 2013.

\bibitem{chazal2021introduction}
Fr{\'{e}}d{\'{e}}ric Chazal and Bertrand Michel.
\newblock An introduction to topological data analysis: Fundamental and
  practical aspects for data scientists.
\newblock {\em Frontiers in Artificial Intelligence}, 4:667963, 2021.

\bibitem{chen2021topological}
Yuzhou Chen, Baris Coskunuzer, and Yulia~R. Gel.
\newblock Topological relational learning on graphs.
\newblock In Marc'Aurelio Ranzato, Alina Beygelzimer, Yann~N. Dauphin, Percy
  Liang, and Jennifer~Wortman Vaughan, editors, {\em Advances in Neural
  Information Processing Systems 34: Annual Conference on Neural Information
  Processing Systems 2021, NeurIPS 2021, December 6-14, 2021, virtual}, pages
  27029--27042, 2021.

\bibitem{chen2022tamp}
Yuzhou Chen, Ignacio Segovia{-}Dominguez, Baris Coskunuzer, and Yulia~R. Gel.
\newblock Tamp-s2gcnets: Coupling time-aware multipersistence knowledge
  representation with spatio-supra graph convolutional networks for time-series
  forecasting.
\newblock In {\em the Tenth International Conference on Learning
  Representations, {ICLR} 2022, Virtual Event, April 25-29, 2022}.
  OpenReview.net, 2022.

\bibitem{vcufar2021fast}
Matija {\v{C}}ufar and {\v{Z}}iga Virk.
\newblock Fast computation of persistent homology representatives with
  involuted persistent homology.
\newblock {\em arXiv preprint arXiv:2105.03629}, 2021.

\bibitem{dey2019persistent}
Tamal~K. Dey, Tao Hou, and Sayan Mandal.
\newblock Persistent 1-cycles: Definition, computation, and its application.
\newblock In Rebeca Marfil, Mariletty Calder{\'{o}}n, Fernando~D{\'{\i}}az del
  R{\'{\i}}o, Pedro Real, and Antonio Bandera, editors, {\em Computational
  Topology in Image Context - 7th International Workshop, {CTIC} 2019,
  M{\'{a}}laga, Spain, January 24-25, 2019, Proceedings}, volume 11382 of {\em
  Lecture Notes in Computer Science}, pages 123--136. Springer, 2019.

\bibitem{dey2022computational}
Tamal~Krishna Dey and Yusu Wang.
\newblock {\em Computational Topology for Data Analysis}.
\newblock Cambridge University Press, 2022.

\bibitem{edelsbrunner2010computational}
Herbert Edelsbrunner and John Harer.
\newblock {\em Computational Topology - an Introduction}.
\newblock American Mathematical Society, 2010.

\bibitem{escolar2016optimal}
Emerson~G Escolar and Yasuaki Hiraoka.
\newblock Optimal cycles for persistent homology via linear programming.
\newblock In {\em Optimization in the Real World}, pages 79--96. Springer,
  2016.

\bibitem{giatsidis2014corecluster}
Christos Giatsidis, Fragkiskos~D. Malliaros, Dimitrios~M. Thilikos, and
  Michalis Vazirgiannis.
\newblock Corecluster: {A} degeneracy based graph clustering framework.
\newblock In Carla~E. Brodley and Peter Stone, editors, {\em Proceedings of the
  Twenty-Eighth {AAAI} Conference on Artificial Intelligence, July 27 -31,
  2014, Qu{\'{e}}bec City, Qu{\'{e}}bec, Canada}, pages 44--50. {AAAI} Press,
  2014.

\bibitem{giatsidis2011evaluating}
Christos Giatsidis, Dimitrios~M. Thilikos, and Michalis Vazirgiannis.
\newblock Evaluating cooperation in communities with the k-core structure.
\newblock In {\em Proceedings of the International Conference on Advances in
  Social Networks Analysis and Mining, {ASONAM} 2011, Kaohsiung, Taiwan, 25-27
  July 2011}, pages 87--93. {IEEE} Computer Society, 2011.

\bibitem{giunti22}
Barbara Giunti.
\newblock Tda applications library, 2022.
\newblock \url{https://www.zotero.org/groups/2425412/tda-applications/library}.

\bibitem{hatcher2002algebraic}
Allen Hatcher.
\newblock {\em Algebraic Topology}.
\newblock Cambridge University Press, 2002.

\bibitem{hofer2020graph}
Christoph~D. Hofer, Florian Graf, Bastian Rieck, Marc Niethammer, and Roland
  Kwitt.
\newblock Graph filtration learning.
\newblock In {\em Proceedings of the 37th International Conference on Machine
  Learning, {ICML} 2020, 13-18 July 2020, Virtual Event}, volume 119 of {\em
  Proceedings of Machine Learning Research}, pages 4314--4323. {PMLR}, 2020.

\bibitem{hofer2017deep}
Christoph~D. Hofer, Roland Kwitt, Marc Niethammer, and Andreas Uhl.
\newblock Deep learning with topological signatures.
\newblock In Isabelle Guyon, Ulrike von Luxburg, Samy Bengio, Hanna~M. Wallach,
  Rob Fergus, S.~V.~N. Vishwanathan, and Roman Garnett, editors, {\em Advances
  in Neural Information Processing Systems 30: Annual Conference on Neural
  Information Processing Systems 2017, December 4-9, 2017, Long Beach, CA,
  {USA}}, pages 1634--1644, 2017.

\bibitem{horn2021topological}
Max Horn, Edward~De Brouwer, Michael Moor, Yves Moreau, Bastian Rieck, and
  Karsten~M. Borgwardt.
\newblock Topological graph neural networks.
\newblock In {\em Proceedings of the Tenth International Conference on Learning
  Representations, {ICLR} 2022, Virtual Event, April 25-29, 2022}.
  OpenReview.net, 2022.

\bibitem{hu2020open}
Weihua Hu, Matthias Fey, Marinka Zitnik, Yuxiao Dong, Hongyu Ren, Bowen Liu,
  Michele Catasta, and Jure Leskovec.
\newblock Open graph benchmark: Datasets for machine learning on graphs.
\newblock In Hugo Larochelle, Marc'Aurelio Ranzato, Raia Hadsell,
  Maria{-}Florina Balcan, and Hsuan{-}Tien Lin, editors, {\em Advances in
  Neural Information Processing Systems 33: Annual Conference on Neural
  Information Processing Systems 2020, NeurIPS 2020, December 6-12, 2020,
  virtual}, 2020.

\bibitem{ichinomiya2017persistent}
Takashi Ichinomiya, Ippei Obayashi, and Yasuaki Hiraoka.
\newblock Persistent homology analysis of craze formation.
\newblock {\em Physical Review E}, 95(1):012504, 2017.

\bibitem{jiang2022learning}
Tian Jiang, Meichen Huang, Ignacio Segovia{-}Dominguez, Nathaniel~K. Newlands,
  and Yulia~R. Gel.
\newblock Learning space-time crop yield patterns with zigzag persistence-based
  {LSTM:} toward more reliable digital agriculture insurance.
\newblock In {\em Proceedings of the Thirty-Fourth Conference on Innovative
  Applications of Artificial Intelligence, {IAAI} 2022, Virtual Event, February
  22 - March 1, 2022}, pages 12538--12544. {AAAI} Press, 2022.

\bibitem{kahle2009topology}
Matthew Kahle.
\newblock Topology of random clique complexes.
\newblock {\em Discrete mathematics}, 309(6):1658--1671, 2009.

\bibitem{kannan2019persistent}
Harish Kannan, Emil Saucan, Indrava Roy, and Areejit Samal.
\newblock Persistent homology of unweighted complex networks via discrete morse
  theory.
\newblock {\em Scientific Reports}, 9(1):1--18, 2019.

\bibitem{KKMMN2016}
Kristian Kersting, Nils~M. Kriege, Christopher Morris, Petra Mutzel, and Marion
  Neumann.
\newblock Benchmark data sets for graph kernels, 2016.
\newblock \url{http://graphkernels.cs.tu-dortmund.de}.

\bibitem{kovacev2016using}
Violeta Kovacev-Nikolic, Peter Bubenik, Dragan Nikoli{\'c}, and Giseon Heo.
\newblock Using persistent homology and dynamical distances to analyze protein
  binding.
\newblock {\em Statistical Applications in Genetics and Molecular Biology},
  15(1):19--38, 2016.

\bibitem{leibon2008topological}
Gregory Leibon, Scott Pauls, Daniel Rockmore, and Robert Savell.
\newblock Topological structures in the equities market network.
\newblock {\em Proceedings of the National Academy of Sciences},
  105(52):20589--20594, 2008.

\bibitem{snapnets}
Jure Leskovec and Andrej Krevl.
\newblock {SNAP Datasets}: {Stanford} large network dataset collection.
\newblock \url{http://snap.stanford.edu/data}, June 2014.

\bibitem{malott2019fast}
Nicholas~O Malott and Philip~A Wilsey.
\newblock Fast computation of persistent homology with data reduction and data
  partitioning.
\newblock In {\em Big Data}, pages 880--889. IEEE, 2019.

\bibitem{mcauley2012learning}
Julian~J. McAuley and Jure Leskovec.
\newblock Learning to discover social circles in ego networks.
\newblock In Peter~L. Bartlett, Fernando C.~N. Pereira, Christopher J.~C.
  Burges, L{\'{e}}on Bottou, and Kilian~Q. Weinberger, editors, {\em Advances
  in Neural Information Processing Systems 25: 26th Annual Conference on Neural
  Information Processing Systems 2012. Proceedings of a meeting held December
  3-6, 2012, Lake Tahoe, Nevada, United States}, pages 548--556, 2012.

\bibitem{mischaikow2013morse}
Konstantin Mischaikow and Vidit Nanda.
\newblock Morse theory for filtrations and efficient computation of persistent
  homology.
\newblock {\em Discrete \& Computational Geometry}, 50(2):330--353, 2013.

\bibitem{nielson2015topological}
Jessica~L Nielson, Jesse Paquette, Aiwen~W Liu, Cristian~F Guandique, C~Amy
  Tovar, Tomoo Inoue, Karen-Amanda Irvine, John~C Gensel, Jennifer Kloke,
  Tanya~C Petrossian, et~al.
\newblock Topological data analysis for discovery in preclinical spinal cord
  injury and traumatic brain injury.
\newblock {\em Nature Communications}, 6(1):1--12, 2015.

\bibitem{nikolentzos2018degeneracy}
Giannis Nikolentzos, Polykarpos Meladianos, Stratis Limnios, and Michalis
  Vazirgiannis.
\newblock A degeneracy framework for graph similarity.
\newblock In J{\'{e}}r{\^{o}}me Lang, editor, {\em Proceedings of the
  Twenty-Seventh International Joint Conference on Artificial Intelligence,
  {IJCAI} 2018, July 13-19, 2018, Stockholm, Sweden}, pages 2595--2601.
  ijcai.org, 2018.

\bibitem{obayashi2018volume}
Ippei Obayashi.
\newblock Volume-optimal cycle: Tightest representative cycle of a generator in
  persistent homology.
\newblock {\em SIAM Journal on Applied Algebra and Geometry}, 2(4):508--534,
  2018.

\bibitem{ofori2021topological}
Dorcas Ofori{-}Boateng, Ignacio Segovia{-}Dominguez, Cuneyt~Gurcan Akcora,
  Murat Kantarcioglu, and Yulia~R. Gel.
\newblock Topological anomaly detection in dynamic multilayer blockchain
  networks.
\newblock In Nuria Oliver, Fernando P{\'{e}}rez{-}Cruz, Stefan Kramer, Jesse
  Read, and Jos{\'{e}}~Antonio Lozano, editors, {\em Machine Learning and
  Knowledge Discovery in Databases. Research Track - European Conference,
  {ECML} {PKDD} 2021, Bilbao, Spain, September 13-17, 2021, Proceedings, Part
  {I}}, volume 12975 of {\em Lecture Notes in Computer Science}, pages
  788--804. Springer, 2021.

\bibitem{otter2017roadmap}
Nina Otter, Mason~A. Porter, Ulrike Tillmann, Peter Grindrod, and Heather~A.
  Harrington.
\newblock A roadmap for the computation of persistent homology.
\newblock {\em {EPJ} Data Sci.}, 6(1):17, 2017.

\bibitem{rieck2019persistent}
Bastian Rieck, Christian Bock, and Karsten~M. Borgwardt.
\newblock A persistent weisfeiler-lehman procedure for graph classification.
\newblock In Kamalika Chaudhuri and Ruslan Salakhutdinov, editors, {\em
  Proceedings of the 36th International Conference on Machine Learning, {ICML}
  2019, 9-15 June 2019, Long Beach, California, {USA}}, volume~97 of {\em
  Proceedings of Machine Learning Research}, pages 5448--5458. {PMLR}, 2019.

\bibitem{nr}
Ryan~A. Rossi and Nesreen~K. Ahmed.
\newblock The network data repository with interactive graph analytics and
  visualization.
\newblock In Blai Bonet and Sven Koenig, editors, {\em Proceedings of the
  Twenty-Ninth {AAAI} Conference on Artificial Intelligence, January 25-30,
  2015, Austin, Texas, {USA}}, pages 4292--4293. {AAAI} Press, 2015.

\bibitem{segovia2021tlife}
Ignacio Segovia{-}Dominguez, Zhiwei Zhen, Rishabh Wagh, Huikyo Lee, and
  Yulia~R. Gel.
\newblock Tlife-lstm: Forecasting future {COVID-19} progression with
  topological signatures of atmospheric conditions.
\newblock In Kamal Karlapalem, Hong Cheng, Naren Ramakrishnan, R.~K. Agrawal,
  P.~Krishna Reddy, Jaideep Srivastava, and Tanmoy Chakraborty, editors, {\em
  Proceedings of the Advances in Knowledge Discovery and Data Mining - 25th
  Pacific-Asia Conference, {PAKDD} 2021, Virtual Event, May 11-14, 2021,
  Proceedings, Part {I}}, volume 12712 of {\em Lecture Notes in Computer
  Science}, pages 201--212. Springer, 2021.

\bibitem{seidman1983network}
Stephen~B Seidman.
\newblock Network structure and minimum degree.
\newblock {\em Social Networks}, 5(3):269--287, 1983.

\bibitem{shanahan2013large}
M.~Shanahan, V.~P. Bingman, T.~Shimizu, M.~Wild, and O.~G{\"u}nt{\"u}rk{\"u}n.
\newblock Large-scale network organization in the avian forebrain: a
  connectivity matrix and theoretical analysis.
\newblock {\em Frontiers in Computational Neuroscience}, 7:89, 2013.

\bibitem{sizemore2017classification}
Ann Sizemore, Chad Giusti, and Danielle~S Bassett.
\newblock Classification of weighted networks through mesoscale homological
  features.
\newblock {\em Journal of Complex Networks}, 5(2):245--273, 2017.

\bibitem{yan2021link}
Zuoyu Yan, Tengfei Ma, Liangcai Gao, Zhi Tang, and Chao Chen.
\newblock Link prediction with persistent homology: An interactive view.
\newblock In Marina Meila and Tong Zhang, editors, {\em Proceedings of the 38th
  International Conference on Machine Learning, {ICML} 2021, 18-24 July 2021,
  Virtual Event}, volume 139 of {\em Proceedings of Machine Learning Research},
  pages 11659--11669. {PMLR}, 2021.

\bibitem{yuvaraj2021topological}
Monisha Yuvaraj, Asim~K Dey, Vyacheslav Lyubchich, Yulia~R Gel, and H~Vincent
  Poor.
\newblock Topological clustering of multilayer networks.
\newblock {\em Proceedings of the National Academy of Sciences},
  118(21):e2019994118, 2021.

\bibitem{zhao2019learning}
Qi~Zhao and Yusu Wang.
\newblock Learning metrics for persistence-based summaries and applications for
  graph classification.
\newblock In Hanna~M. Wallach, Hugo Larochelle, Alina Beygelzimer, Florence
  d'Alch{\'{e}}{-}Buc, Emily~B. Fox, and Roman Garnett, editors, {\em Advances
  in Neural Information Processing Systems 32: Annual Conference on Neural
  Information Processing Systems 2019, NeurIPS 2019, December 8-14, 2019,
  Vancouver, BC, Canada}, pages 9855--9866, 2019.

\bibitem{zhao2020persistence}
Qi~Zhao, Ze~Ye, Chao Chen, and Yusu Wang.
\newblock Persistence enhanced graph neural network.
\newblock In Silvia Chiappa and Roberto Calandra, editors, {\em The 23rd
  International Conference on Artificial Intelligence and Statistics, {AISTATS}
  2020, 26-28 August 2020, Online [Palermo, Sicily, Italy]}, volume 108 of {\em
  Proceedings of Machine Learning Research}, pages 2896--2906. {PMLR}, 2020.

\bibitem{zomorodian2005computing}
Afra Zomorodian and Gunnar Carlsson.
\newblock Computing persistent homology.
\newblock {\em Discrete \& Computational Geometry}, 33(2):249--274, 2005.

\end{thebibliography}

\clearpage
\appendix

\noindent{\huge \textbf{Appendix}}

\noindent{\large \textbf{for Reduction Algorithms for Persistence Diagrams of Networks: CoralTDA and PrunIT}}

This appendix gives the proofs of our theorems, list the pseudocode of CoralTDA and Prunit and show further reduction results in networks. 

\section{Datasets}

The following table lists characteristics of graphs in our datasets.

\begin{table}[hb]
	\centering
	\small
	\caption{Characteristics of the datasets in graph and node classification experiments.}
	\label{tab:dataset}
	\begin{tabular}{l r r r}
		\toprule
		Dataset & NumGraphs& AvgNumNodes & AvgNumEdges\\
		\midrule
		DD&1178&284.32&715.66\\
		DHFR &467&42.43&44.54\\
		ENZYMES & 600  & 32.6 & 62.14\\
		FIRSTMM&41&1377.27&3074.10\\
		NCI1 & 4110  & 29.87 & 32.30\\
		OHSU &79&82.01&199.66\\
		PROTEINS & 1113   & 39.06 & 72.82\\
		REDDIT-BINARY & 2000   & 429.63 & 497.75\\
		%SYN&300&100.00&196.00\\
		SYNNEW&300&100.0&196.25\\
		\midrule
		TWITTER &973&83.5&1817\\
		FACEBOOK&10&403.9&8823.4\\
		\midrule
		CORA & 1 & 2708 & 5429\\
		CITESEER & 1& 3264 & 4536\\
		\midrule 
		ARXIV &169343	&33&111.8\\
		MAG &1939743&31&112.5\\
		\bottomrule
	\end{tabular}
	
\end{table}

\section{Algorithms}
\label{sec:algorithms}

This section gives the CoralTDA reduction (\Cref{alg:coreReduction}) and PrunIT algorithm (\Cref{alg:dominatingSet}). Our pseudocode is optimized for clarity. Our implementation is available at \href{https://github.com/cakcora/PersistentHomologyWithCoralPrunit}{github.com/cakcora/PersistentHomologyWithCoralPrunit}.

\begin{algorithm}[ht]
	\caption{\hspace*{-4.3pt}{.} CoralTDA}
	%$\operatorname{CoralTDA}(\CG,k)$.}
	\label{alg:coreReduction}
	\hspace*{\algorithmicindent} \textbf{Input:} $k$ and $\CG$ \\
	\hspace*{\algorithmicindent} \textbf{Output:} ${\CG}^{k+1}$ 
	\begin{algorithmic}[1]
		\STATE{flag=true} 
		\STATE{$\CG_{k+1} = \CG$}
		\WHILE{flag is true and  $\CG_{k+1}$ is not empty}
		\STATE{flag=false}
		\FOR{$u \in \V$}
		\IF{$\left |\N(u)\right|< (k+1)$}
		\STATE{flag=true}
		\STATE{ $\CG_{k+1} = \CG_{k+1}\setminus {u}$}
		\ENDIF
		\ENDFOR
		\ENDWHILE
		\RETURN $\CG_{k+1}$
	\end{algorithmic}
\end{algorithm}

\begin{algorithm}[ht]
	\caption{\hspace*{-4.3pt}{.} PrunIT}
	%$\operatorname{DominatedSet}(\CG)$.} 
	\label{alg:dominatingSet}
	\hspace*{\algorithmicindent} \textbf{Input:} $\CG=(\V,\E)$ \\
	\hspace*{\algorithmicindent} \textbf{Output:} $\CG'$ 
	\begin{algorithmic}[1]
		\STATE{iterate=true}
		\STATE{$\CG^\prime = \CG$}
		\WHILE{iterate is true}
		\STATE{iterate=false}
		\FOR{$u \in \V$}
		\FOR{$w \in \N(u)$}
		\STATE{$dom_{u\rightarrow w}=dom_{w\rightarrow u}=false$}
		\FOR{$n \in \{\N(u)\cup\N(w)\}$}
		\IF{$e_{un}\in \E \land e_{wn} \notin \E$}
		\STATE{$dom_{u\rightarrow w}=true$}
		\ENDIF
		\IF{$e_{uw}\notin \E \land e_{vw} \in \E$}
		\STATE{$dom_{w\rightarrow u}=true$}
		\ENDIF
		\ENDFOR
		\IF{$dom_{u\rightarrow w}=true \land dom_{w\rightarrow u}=false$}
		\STATE{ $\CG^\prime = \CG^\prime\setminus {w}$ and $iterate=true$}
		\ELSIF{$dom_{u\rightarrow w}=false \land dom_{w\rightarrow u}=true$}
		\STATE{ $\CG^\prime = \CG^\prime\setminus {u}$ and $iterate=true$}
		\ENDIF
		
		\ENDFOR
		\ENDFOR
		\ENDWHILE
		\RETURN $\CG^\prime$
	\end{algorithmic}
\end{algorithm}

In \Cref{alg:coreReduction}, we repeatedly delete vertices that have less than k+1 edges and return the resulting graph. In \Cref{alg:dominatingSet}, we search for dominated neighbors of a vertex (lines 7--14) and delete such a neighbor (if exists) from the graph. Here $\N(u)$ denotes the neighbors of $u$ (excluding $u$). We continue until we cannot find a dominated vertex. Lines 4--24  has a complexity of $\mathcal{O}(\left |V\right |^3)$. The outer loop (line 3) may run $\left |V\right |$, however we can parallelize the algorithm for each vertex.

\section{Proofs of the Theorems} \label{sec:proofs}

\subsection{Proof of CoralTDA}

First, we give the proof of Theorem \ref{thm:kcores}. We  use the same notation introduced in Section~\ref{sec:coresanddiagrams}.

\noindent {\bf Theorem~\ref{thm:kcores}:} \textit{Let $\CG$ be an unweighted connected graph. Let $f:\V\to\R$ be a filtering function on $\CG$. Let $PD_k(\CG,f)$ represent the $k^{th}$ persistence diagram for the sublevel filtration of the clique complexes. Let $\wh{\CG}^k$ be the $k$-core of $\CG$. Then, for any $j\geq k$
	$$PD_j(\CG,f)=PD_j(\CG^{k+1},f)$$ 	
}

\begin{proof} For simplicity, we prove the theorem for $j=k$. Then, we give the generalization to $j>k$ case. Fix $k\geq 1$. In order to show the theorem, we need to prove that $(b,d)\in PD_k(\wh{\CG})$ if and only if $(b,d)\in PD_k(\wh{\CG}^{k+1})$. In other words, if a $k^{th}$-homology class $\sigma=[S]$ is born at $\wh{\CG}_b$ and dies at $\wh{\CG}_d$, then the same homology class $[S]$ is born at $\wh{\CG}^{k+1}_b$ and dies at $\wh{\CG}^{k+1}_d$.

	We first prove that the inclusion map $\wh{\CG}_i^{k+1}\hookrightarrow\wh{\CG}_i$ in the Diagram \ref{eqn1} induces an isomorphism for the homology groups $H_k(\wh{\CG_i})=H_k(\wh{\CG}_i^{k+1})$ for any $0\leq i\leq m$. Then, the proof of the theorem will follow by the equivalence of the induced persistence modules.
	
	For simplicity, we omit the subscript $i$. Assume $\sigma\in H_k(\wh{\CG})$. In particular, $\sigma$ is a $k$-homology class of $\wh{\CG}$. This means there exists a $k$-cycle $S$ in the chain complex $C_k(\wh{\CG})$ with $\sigma=[S]$. We claim that any vertex $v$ in $S$ must have degree at least $k+1$, and hence $S\in C_k(\wh{\CG}^{k+1})$.

	As $S$ is $k$-cycle $\partial_k S=0\in C_{k-1}(\wh{\CG})$. Let $S=\sum_{i=1}^Nc_i\Delta_i$ where $\{\Delta_i\}$ are $k$-simplices in the simplicial complex $\wh{\CG}$. Then, $\partial S=\partial\sum_ic_i\Delta_i=\sum_ic_i\partial\Delta_i$. 
	Notice that if $\Delta_i \subset \wh{\CG}$, this implies any vertex of $\Delta_i=[w^i_0,w^i_1,\dots,w^i_k]$ has at least degree $k$ in $\CG$ as $\Delta_i$ has $k+1$ vertices $\{w^i_0,w^i_1,\dots,w^i_k\}$, and they are all pairwise connected by an edge in $\CG$. 
	Now, $\partial_k\Delta_i=\sum_{r=0}^{k}(-1)^r\Omega^i_r$ where $\Omega^i_r=[w^i_0,w^i_1,\dots, w^i_{r-1}, w^i_{r+1},\dots,w^i_k]$ are $k-1$ simplices in $\wh{\CG}$ for $0\leq r\leq k$, i.e. $\Omega^i_r$ is a $(k-1)$-face of $\Delta_i$.
	%$(j-1)$-simplex spanned by all vertices in $\Delta_i$ except $w^i_r$. 	

	As $S$ being a $k$-cycle, and $\partial_k S=0$, the sum $\sum_ic_i\partial\Delta_i=0$. This implies that in the sum any $(k-1)$-chain $\Omega^i_r$ must cancel out with another $\Omega^j_s\subset \partial \Delta_j$ where $\Delta_j$ contains $\Omega^i_r$ and a vertex $w^j_s$ which is not a vertex in $\Delta_i$. Therefore, any vertex in $\Delta_i$ is connected to both all other vertices $\{w_r^i\}$ in $\Delta_i$ and another vertex $w_s^j\in \Delta_j\subset S$. Hence, any vertex in $\Delta^i$ has degree $k+1$. This can be generalized to any vertex in $S$. Hence, any vertex in $S$ has degree at least $k+1$. This proves by induction that $S\subset \wh{\CG}^{k+1}$  as follows. All vertices in $S$ has degree $k+1$ in $\CG$. Notice that the vertices of $S$ has degree $\geq k+1$ \underline{because of the other vertices in $S$}, not with the help of outsider vertices. So, as long as all the vertices in $S$ are still in $i$-core $\CG^i$, then any vertex of $S$ will still have degree  $\geq k+1$ in $\CG^i$. Therefore, By going inductively on the core index $i$, none of the vertices of $S$ are removed from $\CG^i$ for $1\leq i\leq k+1$. Therefore, $S$ is a $k$-cycle in $\wh{\CG}^{k+1}$, i.e. $S\in C_k(\wh{\CG}^{k+1})$. 
	
	This proves for $\partial_k:C_{k}(\wh{\CG})\to C_{k-1}(\wh{\CG})$ and $\partial^{k+1}_k: C_{k}(\wh{\CG}^{k+1})\to C_{k-1}(\wh{\CG}^{k+1})$, we have $\mbox{ker}\partial_k=\mbox{ker}\partial_k^{k+1}$ (Recall $\wh{\CG}^{k+1}\subset \wh{\CG}$). Notice that any vertex in any $(k+1)$-simplex in $\wh{\CG}$ has degree at least $k+1$. Then, $C_{k+1}(\wh{\CG})=C_{k+1}(\wh{\CG}^{k+1})$. This implies $\mbox{im}\partial_{k+1}=\mbox{im}\partial_{k+1}^{k+1}$ where $\partial^{k+1}_{k+1}: C_{k+1}(\wh{\CG}^{k+1})\to C_{k}(\wh{\CG}^{k+1})$. Then, $H_k(\wh{\CG})=H_k(\wh{\CG}^{k+1})$. This proves any $k$-cycle $S$ must be produced by the $k$-simplices in $\wh{\CG}^{k+1}$, and lower degree vertices cannot belong to $S$. 
	
	Also,  $k$-core of $\CG_i$ is equal to $i^{th}$ step of the filtration induced by $f:\V^k \to \R$, i.e.$(\CG^k)_i=(\CG_i)^k$. Hence, if $\sigma$ is born in $H_k(\wh{\CG}_b)$, then it is also born in $H_k(\wh{\CG}_b^{k+1})$. If it dies in $H_k(\wh{\CG}_d)$, it also dies in $H_k(\wh{\CG}_d^{k+1})$. This proves $PD_k(\CG)=PD_k(\CG^{k+1})$.
	
	For the generalization in $j>k$ case, one only needs to consider the same process for higher order cycles. For some $j>k$, let $S$ be a $j$-cycle in $\wh{\CG}$, i.e. $C_j(\wh{\CG})$. Then, by above any vertex in $S$ must have degree at least $j+1$. This implies $S$ must be $j$-cycle in $\wh{\CG}^{k+1}$, too. 
	
	Similarly, any vertex in a  $(j+1)$-simplex must have degree $\geq k+1$. This means $C_{j+1}(\wh{\CG})$ $=C_{j+1}(\wh{\CG}^{k+1})$. This implies $H_j(\wh{\CG})=H_j(\wh{\CG}^{k+1})$ for any $j\geq k$. This proves that the inclusion $\wh{\CG}_i^{k+1}\hookrightarrow\wh{\CG}_i$ induces an isomorphism for the homology groups $H_j(\wh{\CG_i})=H_j(\wh{\CG}_i^{k+1})$ for any $j\geq k$. This shows the equivalence of the induced persistence modules. The proof of the theorem follows. 
\end{proof}

\subsection{Proof of PrunIT} \label{sec:proof_Prunit}

Now, we give the proof of Theorem~\ref{thm:reduction}. We  use the same notation given in Section~\ref{sec:trim}.

\noindent {\bf Theorem~\ref{thm:reduction}:} Let $\CG=(\V,\E)$ be an unweighted graph, and $f:\V\to\R$ be a filtering function. Let $u\in\V$ be dominated by $v\in \V$ and $f(u)\geq f(v)$. Then, removing $u$ from $\CG$ does not change the persistence diagrams for sublevel filtration, i.e. for any $k\geq 0$ $$PD_k(\CG,f)=PD_k(\CG-\{u\},f).$$

\begin{proof} Let $\wh{\CG}_0\subset \wh{\CG}_1\subset \wh{\CG}_2\subset \dots\subset \wh{\CG}_m$ be the sequence of clique complexes in the induced sublevel filtration. Let $\alpha_{i_0-1}<f(v)\leq \alpha_{i_0}$.  This means for any $j\geq i_0$, $v\in \CG_j$, and hence $v$ is a vertex in $\wh{\CG}_j$.  Since $f(u)\geq f(v)$, $\alpha_{i_1-1}<f(u)\leq \alpha_{i_1}$ for some $i_1\geq i_0$. Similarly, this implies $u$ first appears in $\CG_{i_1}$, and  $u\in\wh{\CG}_j$ for $j\geq i_1$. In particular, $v$ belongs to all $\CG_j$ containing $u$. 
	
	Let $\CG'=\CG-\{u\}$. Define the sublevel filtration $\wh{\CG}'_0\subset \wh{\CG}'_1\subset \wh{\CG}'_2\subset ...\subset \wh{\CG}'_m$ for the same filtering function (restricted to $\V'$) $f:\V'\to\R$ with the same threshold set $\I$. 
	
	For any $j<i_1$, $\wh{\CG}_j=\wh{\CG}'_j$ as $u\not\in \CG_j$. Fix $j\geq i_1$. Then, $v\in \wh{\CG}_j$ as $f(u)\geq f(v)$. As $v$ dominates $u$ in $\CG$, $v$ dominates $u$ in $\CG_j$, too. Then, $\CG_j$ folds onto $\CG_j'=\CG_j-\{u\}$. Then, by Lemma \ref{lem:folding}, we have homotopy equivalence $\wh{\CG}_j\sim\wh{\CG}'_j$. Hence, for any $j\geq i_1$, this gives $\wh{\CG}_j\sim\wh{\CG}'_j$. 
	
	Recall that for any $j< i_1$,  $\wh{\CG}_j=\wh{\CG}'_j$. Therefore, for any $0\leq j\leq m$, the inclusion $\wh{\CG}'_j\hookrightarrow\wh{\CG}_j$ induces an isomorphism between the homology groups $H_k(\wh{\CG}'_j)\simeq H_k(\wh{\CG}_j)$. This proves the equivalence of the induced persistence modules, and hence the corresponding persistence diagrams, i.e. $PD_k(\CG)=PD_k(\CG')$ for any $k\geq 0$. The proof follows.
\end{proof}

Now, we prove Theorem~\ref{thm:Prunit_power}. Recall that $n^{th}$ power of graph $\CG=(\V,\E)$ is defined as $\CG^n=(\V,\E_n)$ Where $\E_n=\{e_{uv}\mid d(u,v)\leq n\}$. In other words, $\E_n$ is obtained by adding new edges to $\E$ connecting all the vertices whose graph distance $\leq n$. Then, the power filtration of $\CG$ is defined as $\wh{\CG}^0\subset \wh{\CG}^1\subset \wh{\CG}2\subset \dots \subset \wh{\CG}^N$ where $\wh{\CG}^n$ is the clique complex of $\CG^n$, $n^{th}$-power of $\CG$ and  $\wh{\CG}^0$ represent the vertex set $\V$~\cite{aktas2019persistence}. 

\noindent {\bf Theorem~\ref{thm:Prunit_power}:} [PrunIt for Power Filtration] Let $\CG=(\V,\E)$ be an unweighted connected graph. Let $\wh{PD}_k(\CG)$ represent $k^{th}$ persistence diagram of $\CG$ with power filtration. Let $u\in\V$ be dominated by $v\in \V$. Then, for any $k\geq 1$, $$\wh{PD}_k(\CG)=\wh{PD}_k(\CG-\{u\}).$$

\begin{proof} Notice that power filtration for a connected graph is trivial in dimension $0$ as all features but one dies at threshold $1$. So we will assume $k\geq 1$. Let $\HH=\CG-\{u\}$. We will show that $\wh{\CG}^n$ is homotopy equivalent to $\wh{H}^n$ for $n\geq 1$.
	
	Let $\{v=w_0, w_1, w_2, \dots, w_m\}$ be all adjacent vertices to $u$ in $\CG$. As $u$ is dominated by $v$, $\{w_1, w_2, \dots, w_m\}$ will be adjacent to $v$, too. We claim that for any vertex $z\neq u$ in $\CG$, $d(z,v)\leq d(z,u)$ where $d(z,z')$ is the length of the shortest path from $z$ to $z'$ and each edge has length $1$. In particular, any shortest path $\gamma$ from $z$ to $u$ must go through one of the adjacent vertices $\{v, w_1, w_2, \dots, w_m\}$. Let $\tau\subset \gamma$ be the segment starting at $z$ and ending with one of these adjacent vertices, say $w_i$. Then, $d(z,u)=\|\gamma\|=\|\tau\|+1$. Now, consider $d(z,v)$. Since $\gamma'=\tau\cup [w_i,v]$ is a path from $z$ to $v$, this implies $d(z,v)\leq \|\tau\|+1=d(z,u)$.
	
	Now, we claim that if $u$ is dominated by $v$ in $\CG$, then $u$ is dominated by $v$ in $\CG^n$ for any $n$. All we need to show that any adjacent vertex to $u$ in $\CG^n$ is also adjacent to $v$. If $z$ is adjacent to $u$ in $\CG^n$, then $d(z,u)\leq n$ in $\CG$ by the definition of $\CG^n$. By above, we have $d(z,v)\leq d(z,u)\leq n$. This implies $z$ is adjacent to $v$, too. Hence, $u$ is dominated by $v$. Furthermore, as $v$ dominates $u$ in $\CG$, no shortest path goes through $u$ unless the endpoint is $u$. Therefore, the distances in $\CG$ and $\HH$ would be same for any two vertices $z,z'\in \V-\{u\}$. Then, by Lemma~\ref{lem:folding}, $\wh{\CG}^n$ homotopy equivalent to $\wh{\HH}^n$ for any $n\geq 1$.
	
	Then, consider $\wh{PD}_k(\CG)$ for $k\geq 1$. For $k\geq 1$, any $k$-dimensional topological feature will be born in $\wh{\CG}^n$ where $n\geq 1$. As $\wh{\CG}^n$ homotopy equivalent to $\wh{\HH}^n$ for any $n\geq 1$, the proof follows.
\end{proof}

\begin{remark}  \label{remark:coralTDA_power} \normalfont [CoralTDA for Power Filtration] Of course, after the previous result, the natural follow-up question is \textit{"Does CoralTDA also extends to power filtrations?"}. Unfortunately, the answer to this question is "No". The reason is by the nature of power filtration, one keeps adding edges to the existing vertices, and the degrees of the vertices increase in each step. This means $k$-cores significantly changes during this process. A simple counterexample can be found in~\cite[Corollary 6.7]{adamaszek2013clique}, where the author completely classifies the topological types of clique complexes of the powers of cyclic graphs $\{C_n\}$, i.e., A graph which forms a circle. e.g., $C_5$ is a connected graph with 5 vertex and 5 edges. By definition, any $3$-core of a cyclic graph is empty. However, Adamazsek's result show that for any dimension $k$, there exists topological features $\wh{PD}_k(C_n)$ is nontrivial for $n\geq 2k+3$. If coralTDA could be extended to power filtrations, then $\wh{PD}_k(C_n)$ would be trivial for any $k\geq 3$.
\end{remark}

\begin{remark} \label{remark:GraphFiltrationLearning} \normalfont [Combining Reduction Algorithms with Graph Filtration Learning] In~\cite{hofer2020graph}, the authors studied a fundamental question, which is to find the “best filtering function” for the given classification problem. Indeed, the “right” filtering function is often the key behind the classification performance. Under the assumption that dominated vertex value is greater than its dominating vertex value, we can invoke our PrunIt result and run the “Graph Filtration Learning” algorithm of \cite{hofer2020graph} on the pruned graph which will lead to computational and (possibly) performance  gains. For node classification tasks, it would be also wise to check the relation between the class differences between dominated and their dominating vertices. With similar reasoning, our CoralTDA method can also be combined with the “Graph Filtration Learning” algorithm and other methods.
\end{remark}

\begin{remark} \label{rmk:strongcollapse} \normalfont [PrunIT vs. Strong Collapse] In \cite{boissonnat2018computing, boissonnat2020edge, boissonnat2021strong}, the authors introduced Strong Collapse method to effectively reduce the computational costs for the persistence diagrams. In this context, Strong Collapse method is defined for any simplicial complex sequence, whereas PrunIT ise specifically defined for the graph setting. The main difference between Strong Collapse and PrunIt is that PrunIT captures the dominated vertices in the graph at once, before filtration while Strong Collapse needs to be applied in each filtration step. In particular, for a given graph $\CG$ and the filtering function $f$, PrunIT detects the dominated vertices in the graph, and by removing them, gives a smaller graph $\CG'$ with the same persistence diagrams $PD(\CG,f)=PD(\CG'f)$. Since the detection of dominated vertices is done in the graph itself, the PrunIt algorithm works very fast. On the other hand, Strong Collapse method is designed to work with flag complexes. Therefore, if one wants to apply Strong Collapse method to compute $PD(\CG,f)$, one needs to go to the filtration sequence (flag complexes) $\wh{\CG}_1\subset \wh{\CG}_2\subset \dots \subset\wh{\CG}_N$, and apply the Strong Collapse method for each flag complex $\wh{\CG}_n$ for $1\leq n\leq N$ in the the filtration sequence. This means one needs to detect dominated vertices for each $\wh{\CG}_n$ and remove them. Therefore, especially if the filtration sequence is long ($N$ is large), PrunIT would be a more efficient method in graph setting.
	
	We run experiments on the on Email-Enron dataset with 36K nodes and 183K edges (see Table~\ref{tab:prunitresults}) for PrunIt and Strong Collapse methods with degree filtering function. We provide the results in the supplementary material. In particular, we tried two different threshold step sizes $\delta_1=4$ and $\delta_2=12$ which give filtration sequence of lengths $N_1=346$ and $N_2=116$ respectively. In Table~\ref{tab:scollapse}, we give the computation times in seconds for PrunIT and Strong Collapse algorithm to eliminate dominated vertices before PH computations. For PH computation step, recall that PH computation is cubic in terms of count of simplices~\cite{otter2017roadmap}. In Table~\ref{tab:scollapse}, we also give these simplex counts for each method. The details can be found in the supplementary material.  
\end{remark}

\begin{table*}[h!]
	\centering
	\caption{\footnotesize Comparison results for PrunIt and Strong Collapse Method for Email-Enron dataset. \label{tab:scollapse}}
	\setlength\tabcolsep{7.1 pt}
	\scriptsize
	
	\begin{tabular}{lcc|cc}
		\toprule
		& \multicolumn{2}{c}{\textbf{Computation time (sec)}} & \multicolumn{2}{c}{\textbf{Simplex Count (million)}} \\
		% &  \cmidrule(lr){2-3}  &  \cmidrule(lr){2-3} \\
		\textbf{Step size} & \textbf{PrunIt} &\textbf{S. Collapse} & \textbf{PrunIT} &\textbf{S.Collapse}\\
		\midrule
		4 &1412 &7014 & 270.2 & 465.2\\ 
		12 &513 &2520 & 90.7 & 155.8\\ 
		
		\bottomrule
	\end{tabular}
\end{table*}

\section{Further Reduction Results}

\subsection{Reduction for Graph and Node Classification Datasets} 
\label{sec:graphnode2}

In the following \Cref{fig:edge,fig:time,fig:complex}, we report further reduction percentages for each $PD_k(\CG)$ for the graph and node classification datasets given in \Cref{tab:dataset}. Here, the x-axes represent the dimension $k$ for $PD_k(\CG)$. Edge and simplex reduction figures closely resemble the vertex reduction in \Cref{fig:vertex}. In time reduction of \Cref{fig:time}, we see three datasets (OHSU, FACEBOOK, TWITTER) with diverging behavior as follows. 

\begin{figure*}[ht]
	\begin{subfigure}{.24\textwidth}
		\centering
		\includegraphics[width=.99\linewidth]{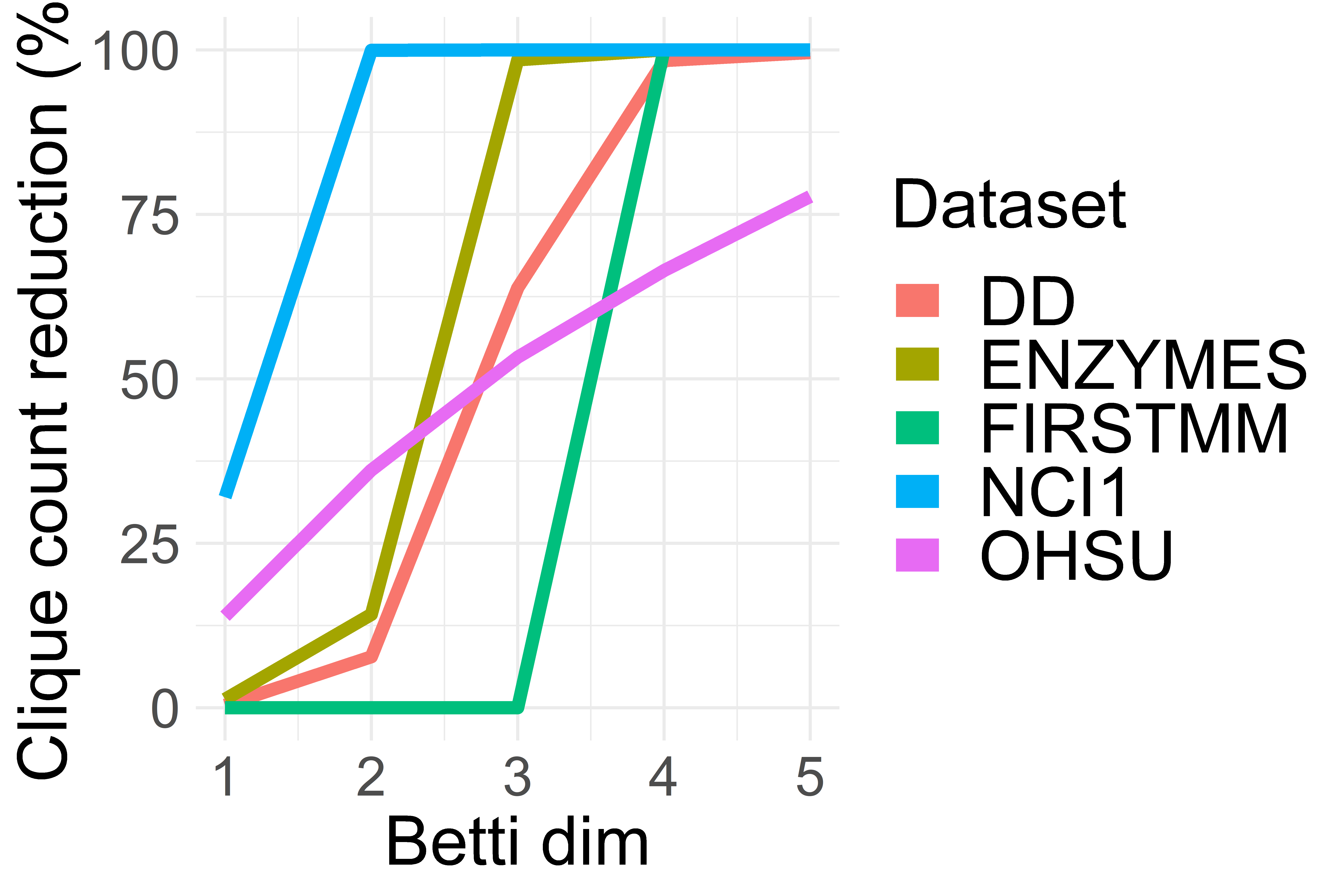}  
	\end{subfigure}
	\begin{subfigure}{.24\textwidth}
		\centering
		\includegraphics[width=.99\linewidth]{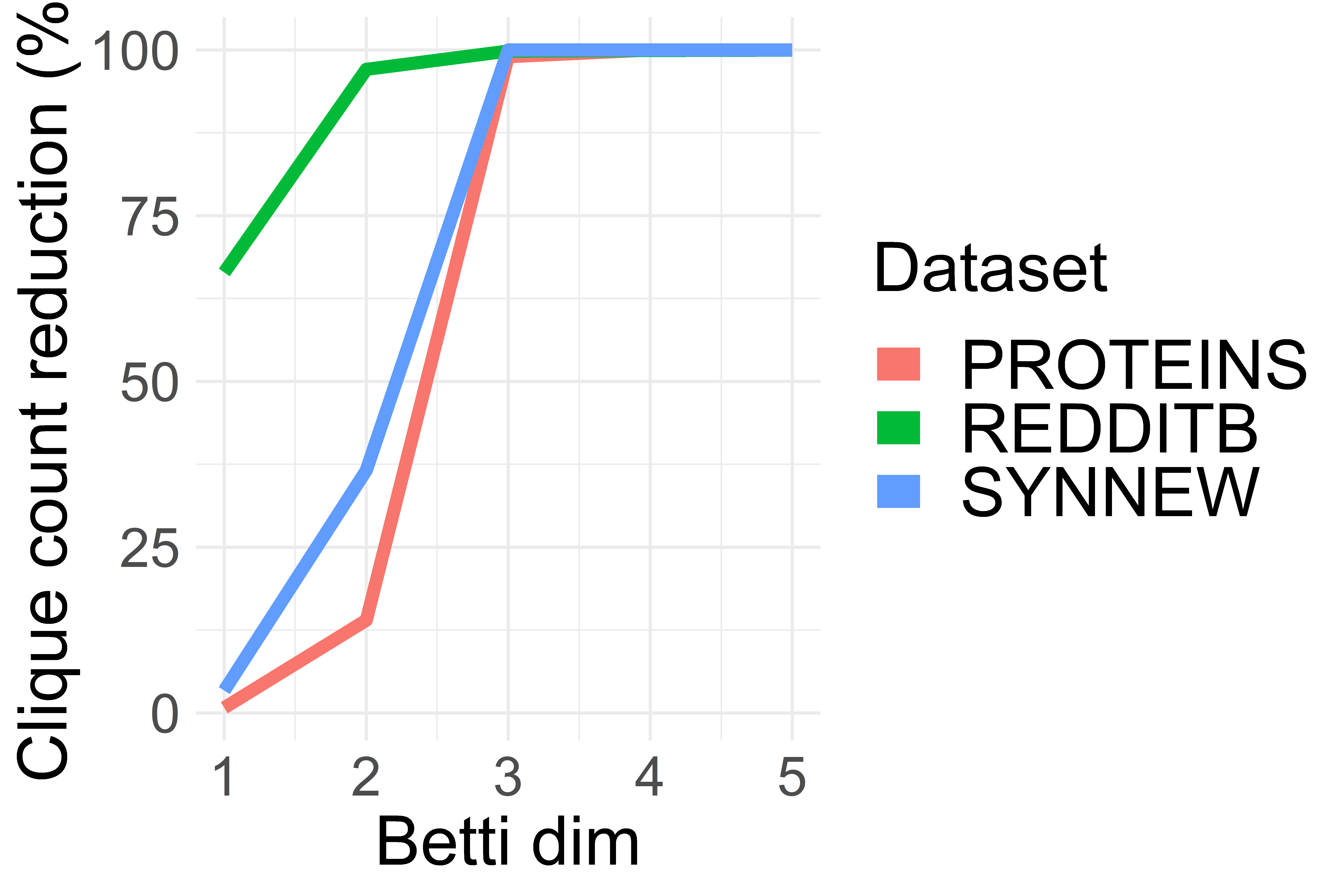}  
	\end{subfigure}
	\begin{subfigure}{.24\textwidth}
		\centering
		\includegraphics[width=.99\linewidth]{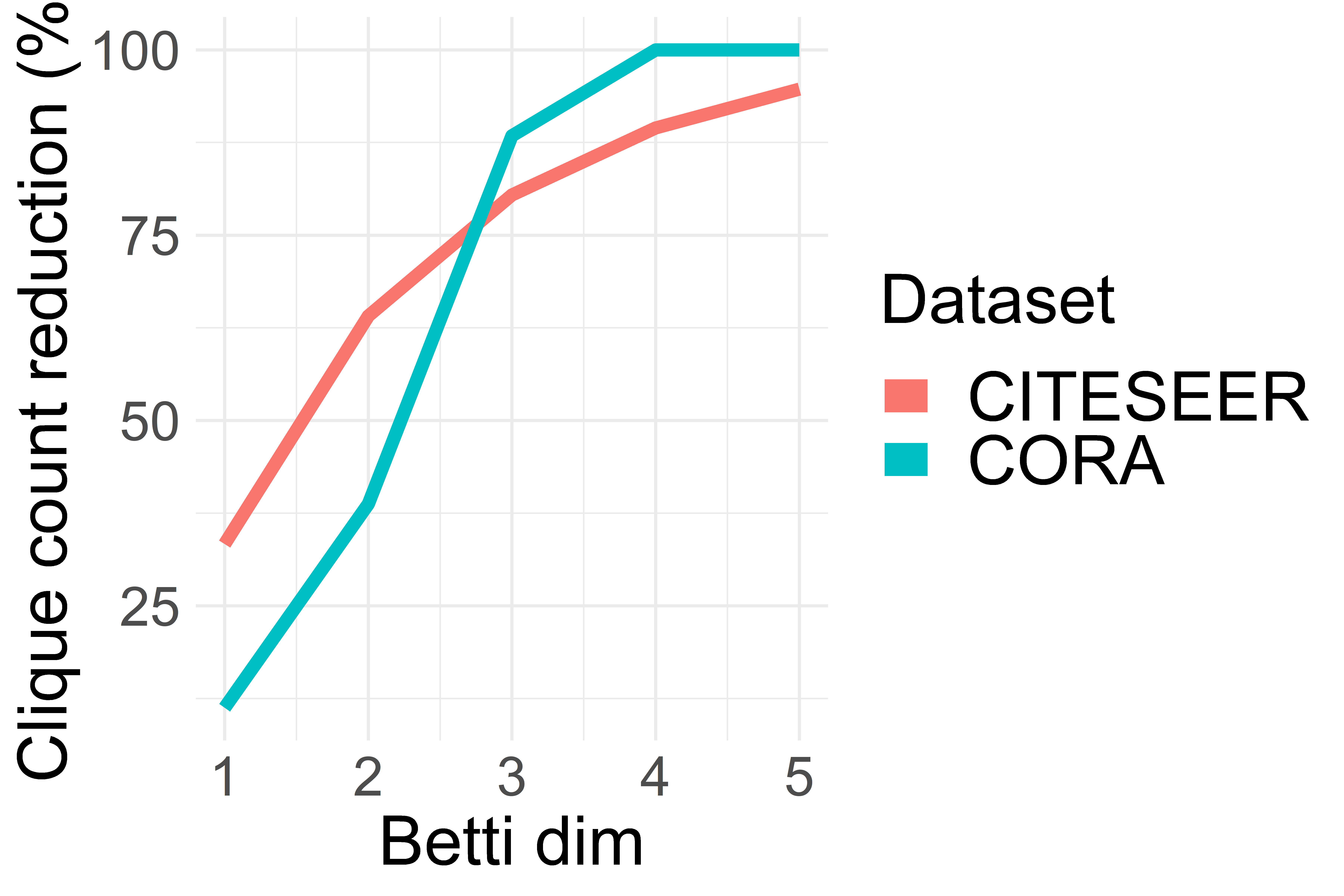}  
	\end{subfigure}
	\begin{subfigure}{.24\textwidth}
		\centering
		\includegraphics[width=.99\linewidth]{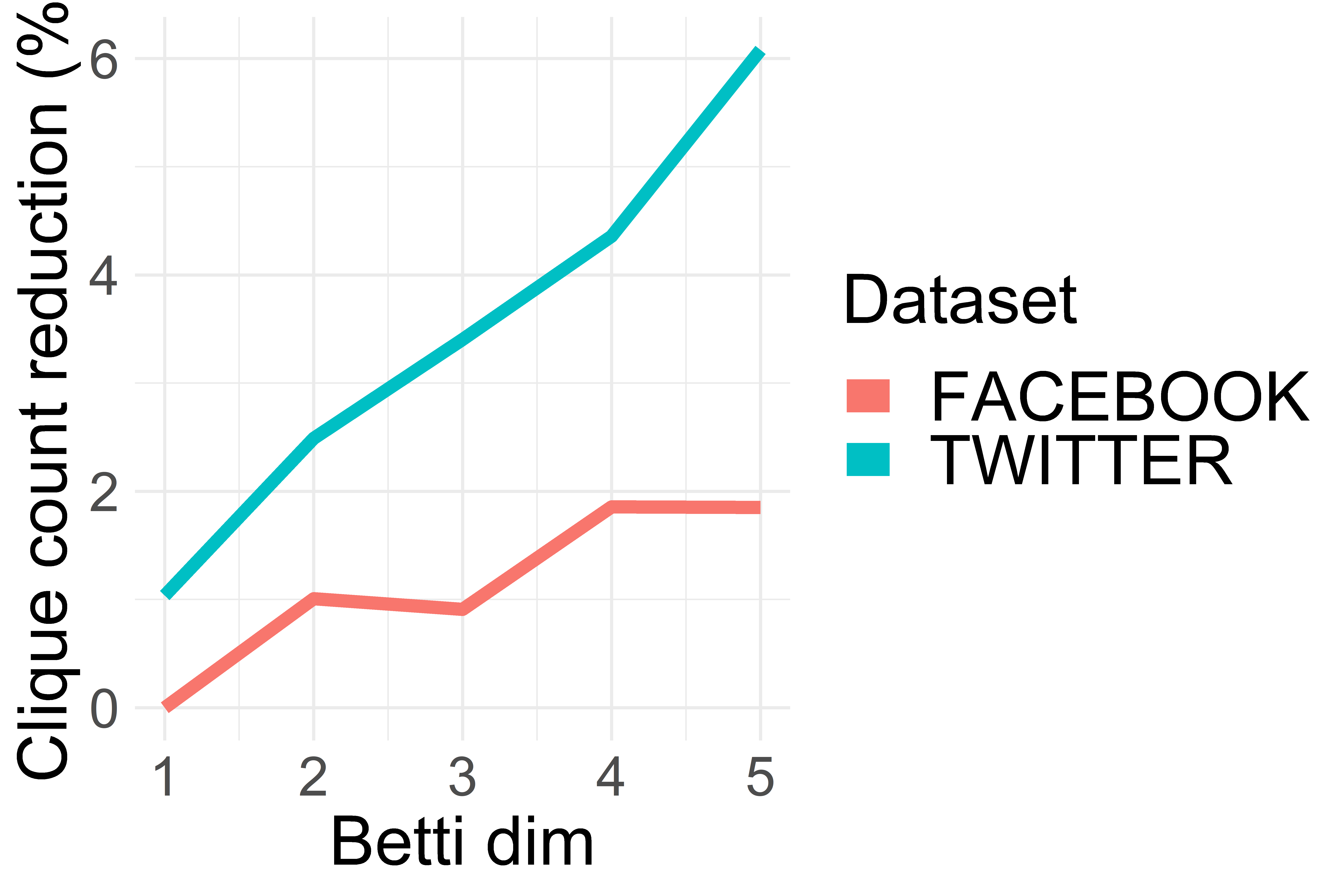}  
	\end{subfigure}
	\caption{CoralTDA clique count reduction in graph and node classification datasets (higher is better).}
	\label{fig:complex}
\end{figure*}

OHSU, despite losing many of vertices in core reduction, does not have a time reduction as much as the other kernel datasets. We explain this behavior with small graph orders but high coreness in OHSU. Note that unlike any other kernel dataset, vertex reduction does not reach 100\% in OHSU graphs. Calling core decomposition on OHSU graphs adds a time cost that takes away the benefit of coralTDA to some extent. As a result, time reduction is at most 25\%.

A similar dynamic is seen in TWITTER and FACEBOOK, where graphs have high cores. As a result, most vertices cannot be peeled away in core reduction (\Cref{fig:vertex} shows at most ~20\% vertex reduction). 

\begin{figure*}[h]
	\begin{subfigure}{.24\textwidth}
		\centering
		\includegraphics[width=.99\linewidth]{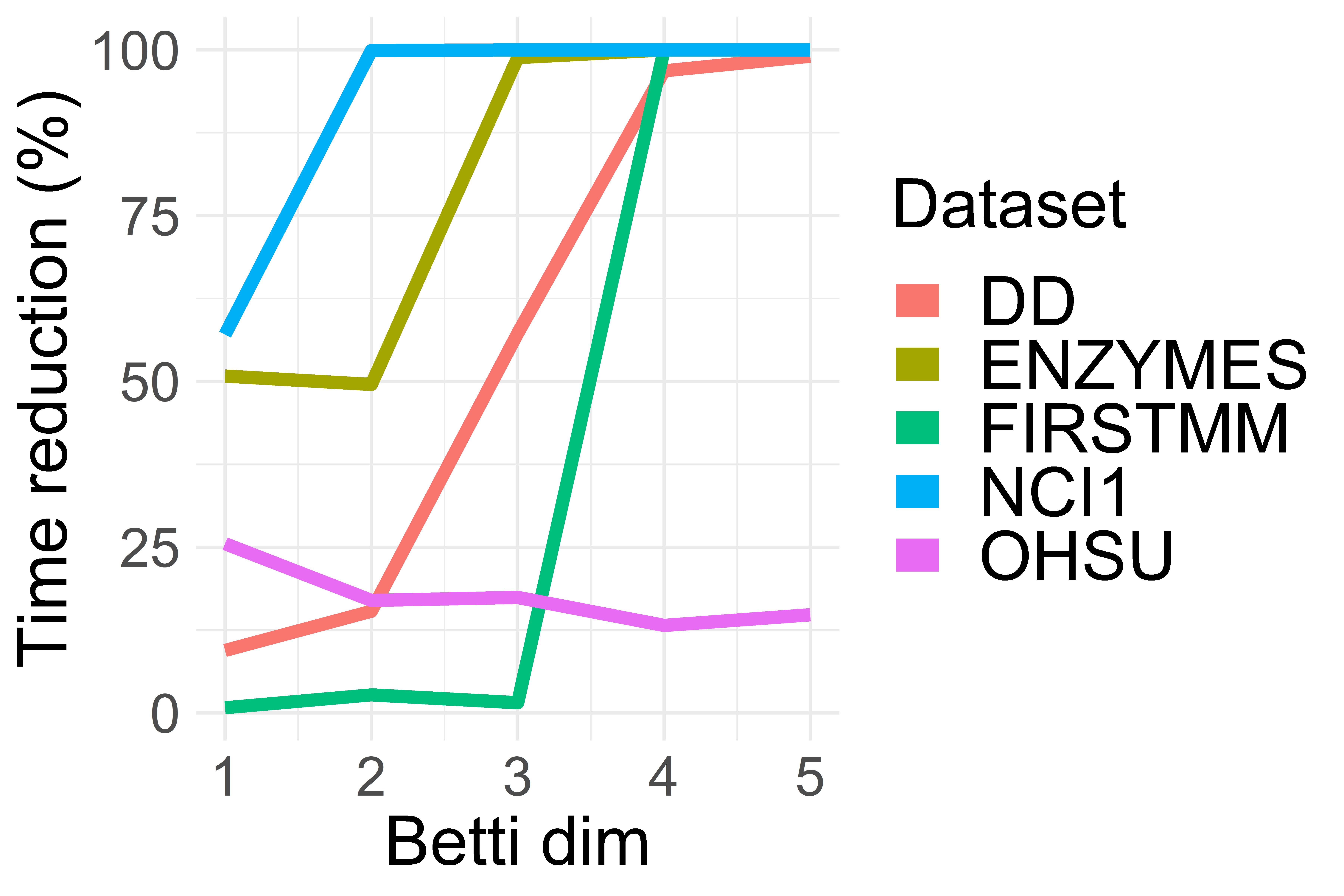}  
	\end{subfigure}
	\begin{subfigure}{.24\textwidth}
		\centering
		\includegraphics[width=.99\linewidth]{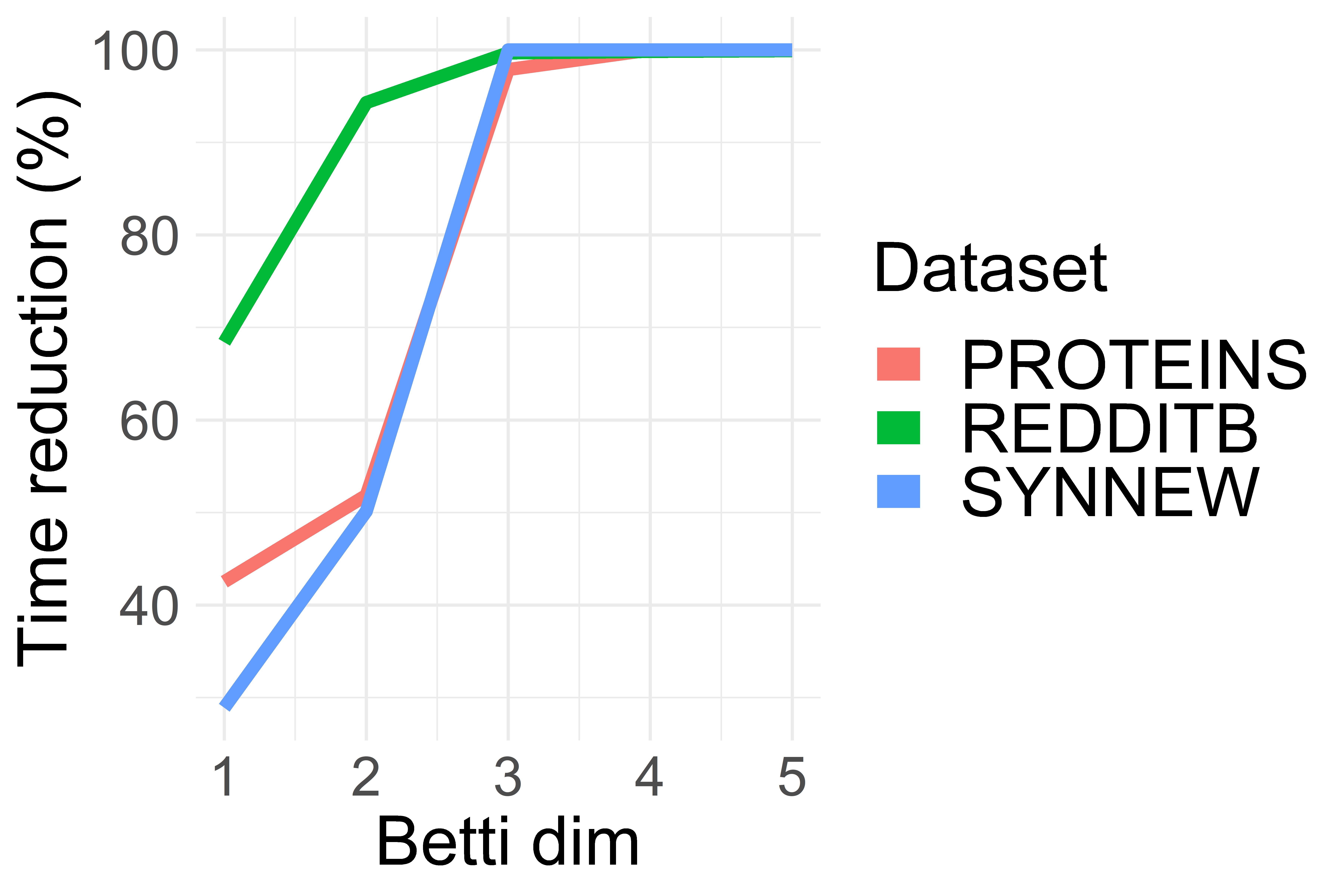}  
	\end{subfigure}
	\begin{subfigure}{.24\textwidth}
		\centering
		\includegraphics[width=.99\linewidth]{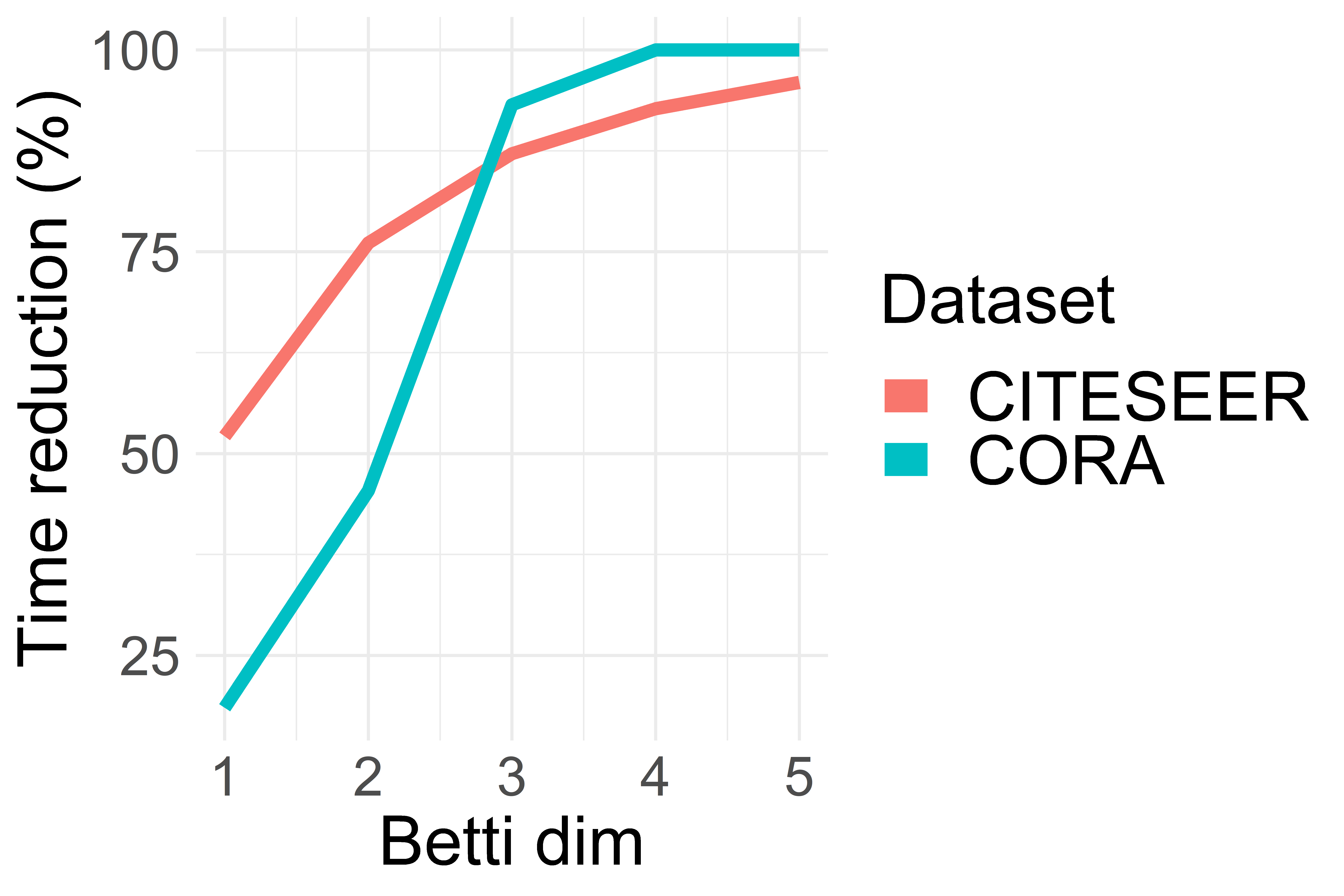}  
	\end{subfigure}
	\begin{subfigure}{.24\textwidth}
		\centering
		\includegraphics[width=.99\linewidth]{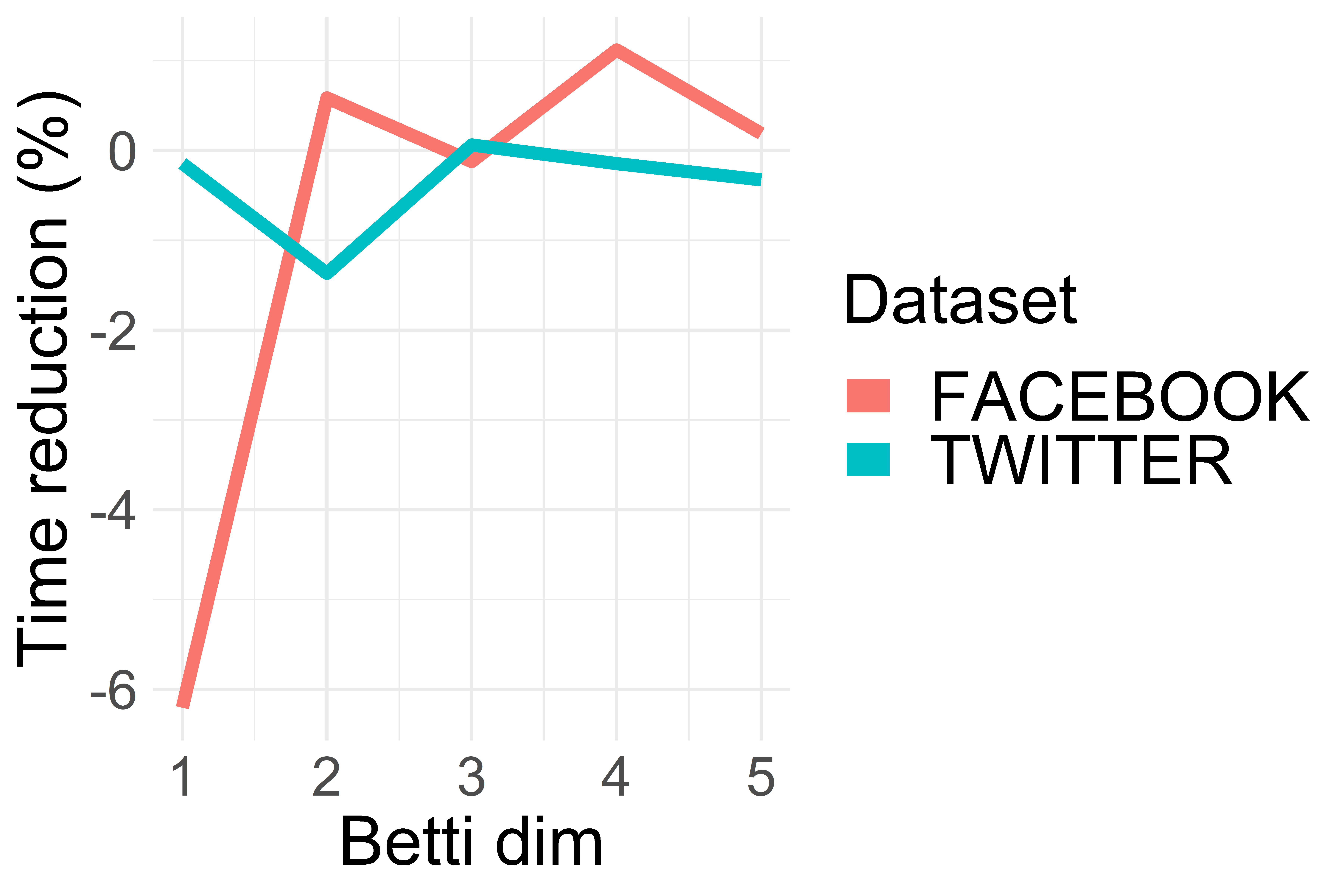}  
	\end{subfigure}
	\caption{CoralTDA time reduction in graph and node classification datasets (higher is better). We do not observe a reduction in Facebook and Twitter datasets. The added computational cost of the algorithm results in negative gains.}
	\label{fig:time}
\end{figure*}

\begin{figure*}[h!]
	\begin{subfigure}{.24\textwidth}
		\centering
		\includegraphics[width=.99\linewidth]{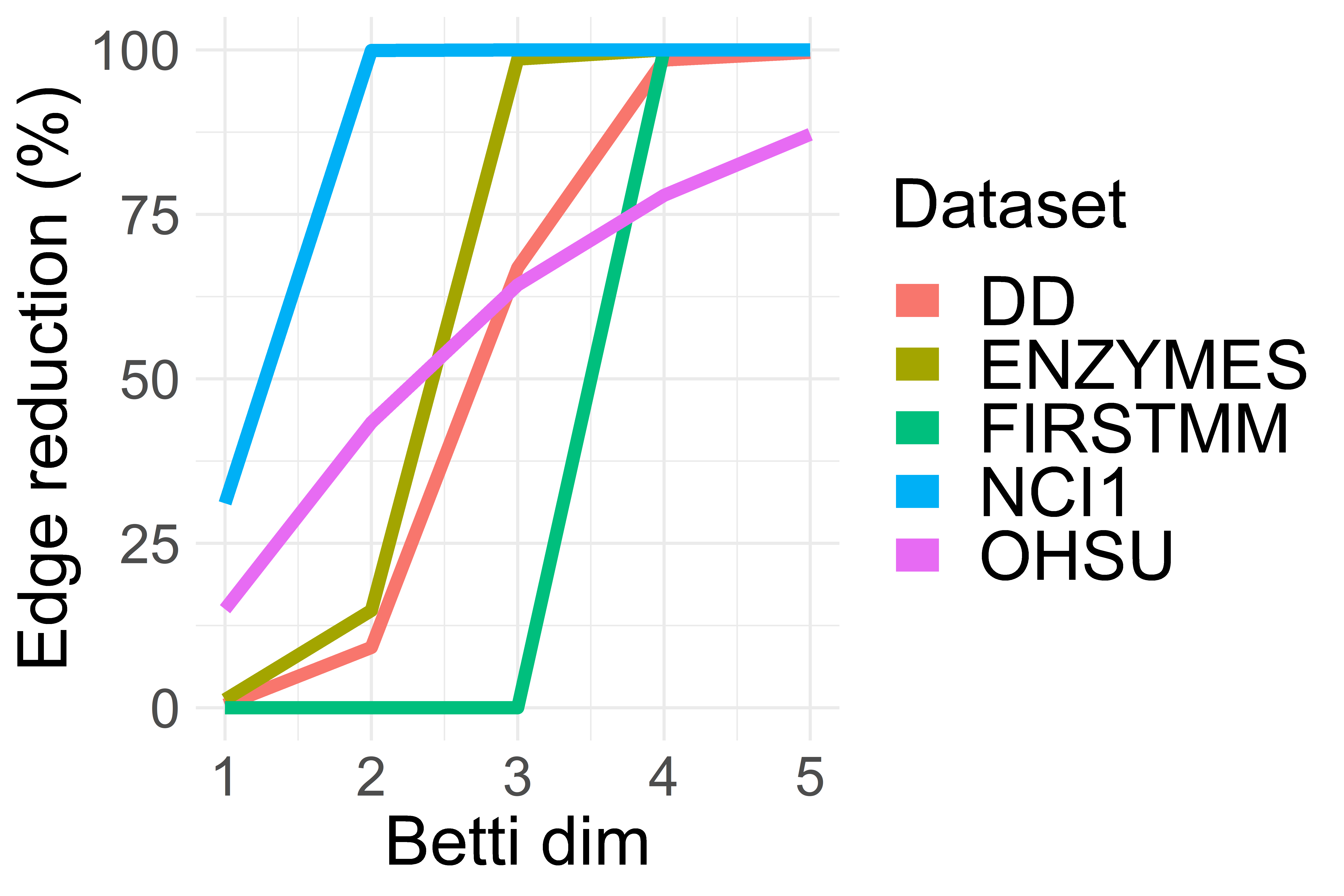}  
	\end{subfigure}
	\begin{subfigure}{.24\textwidth}
		\centering
		\includegraphics[width=.99\linewidth]{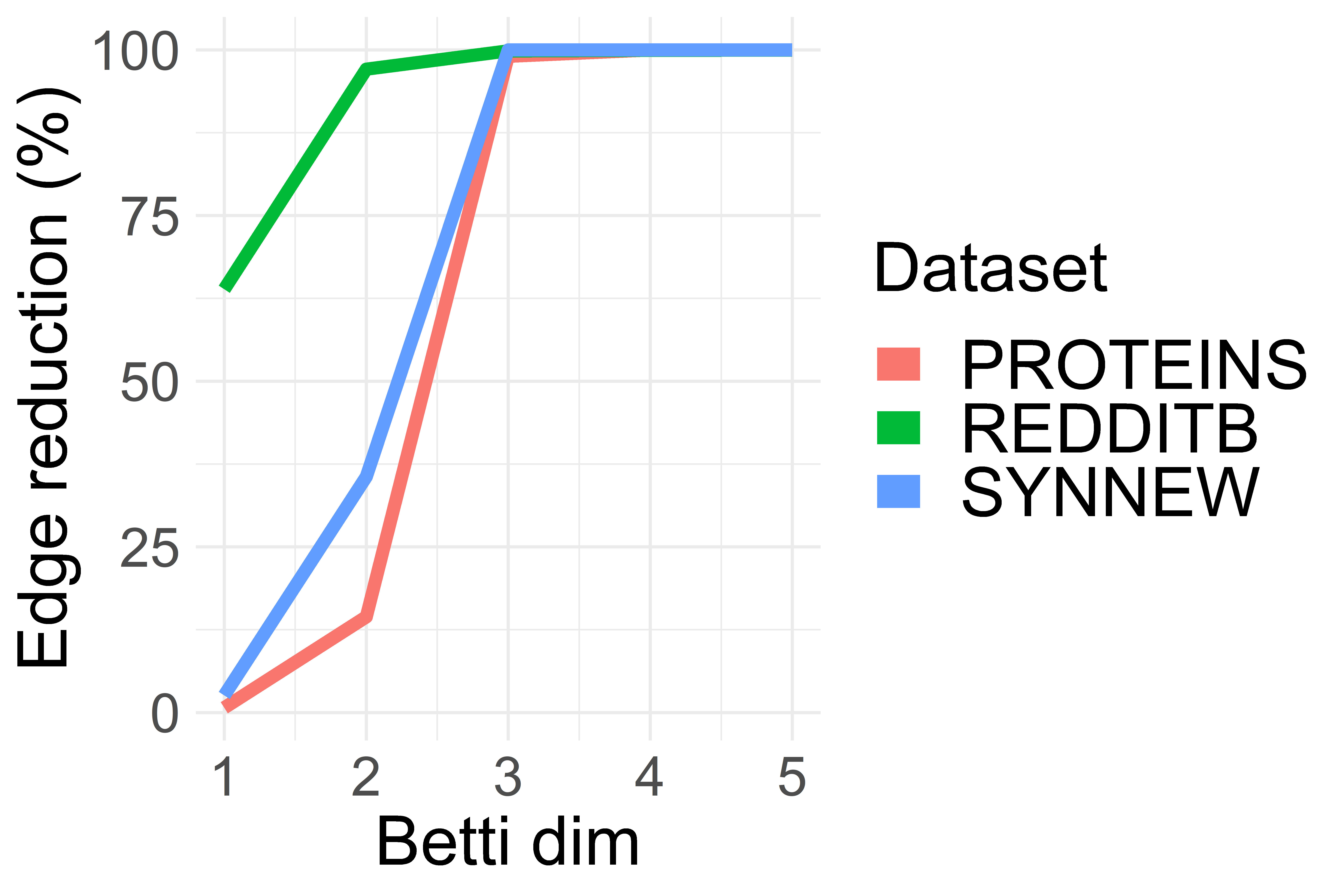}  
	\end{subfigure}
	\begin{subfigure}{.24\textwidth}
		\centering
		\includegraphics[width=.99\linewidth]{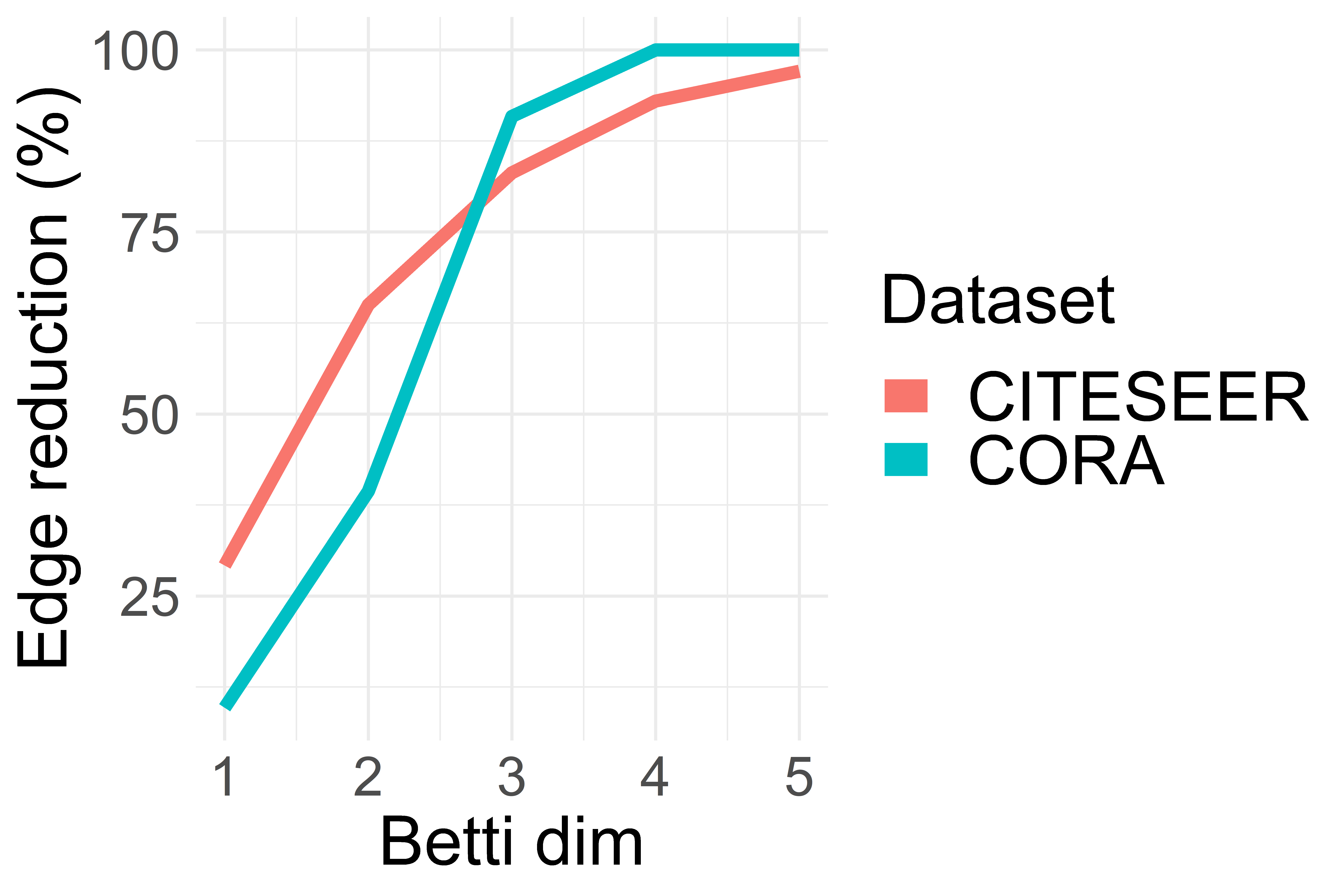}  
	\end{subfigure}
	\begin{subfigure}{.24\textwidth}
		\centering
		\includegraphics[width=.99\linewidth]{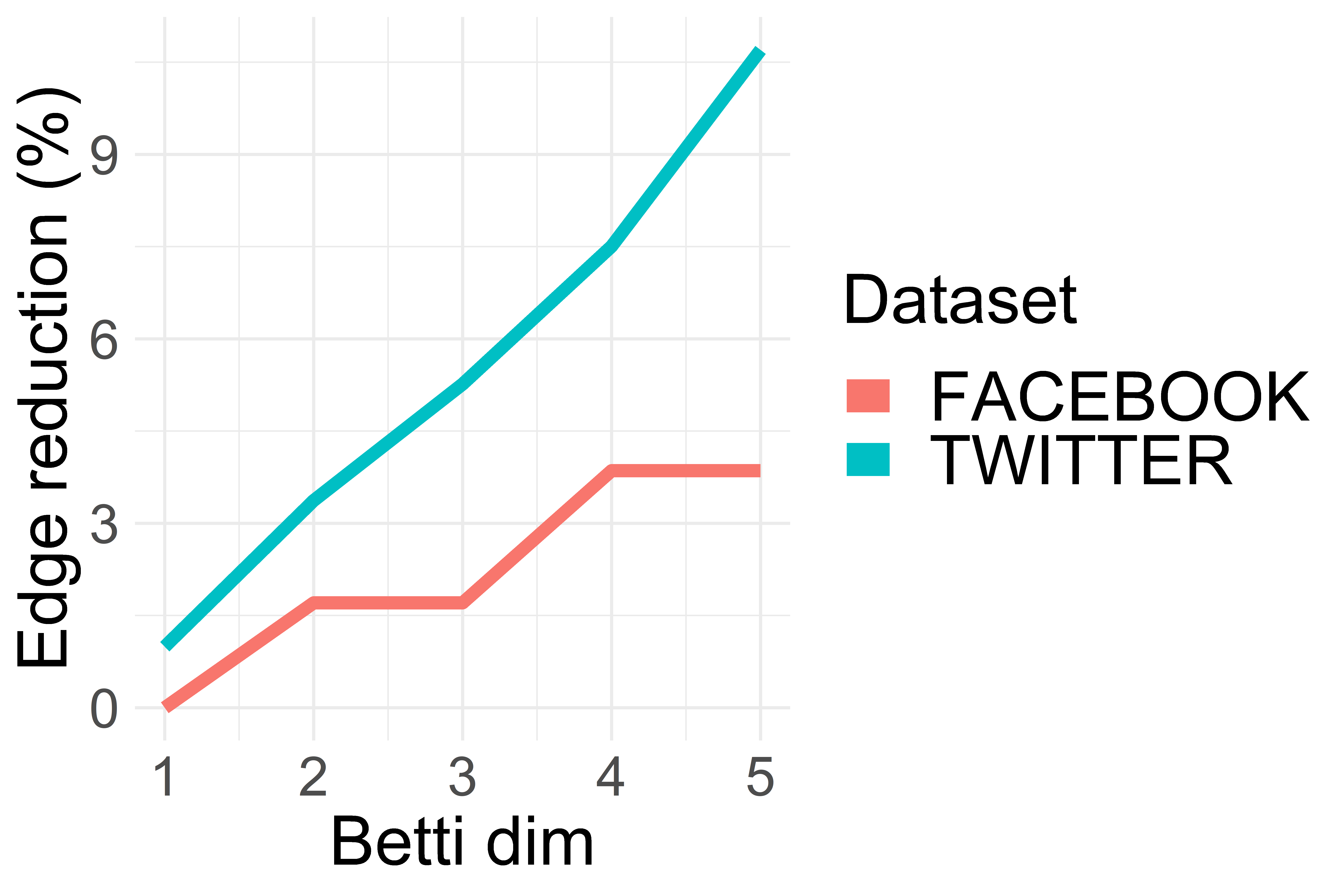}  
	\end{subfigure}
	\caption{CoralTDA edge reduction ($100\times(\left |\E\right |-\left |\E^k\right |) / \left |\E\right |$) in graph and node classification datasets (higher is better). Reduction values are averages from graph instances of the datasets.}
	\label{fig:edge}
\end{figure*}

\subsection{Clustering Coefficient and Higher Persistence Diagrams}
While the proposed reduction algorithms provide a computationally efficient framework for higher-order topological summaries of graphs, in many real-world applications, we still need to address the important question of whether a given network exhibits a nontrivial higher ($k\geq 2)$ persistence diagram $PD_k(\CG)$. When the $(k+1)$-core $\CG^{k+1}$ is sufficiently small, the answer immediately follows from our result. However, the problem is substantially more challenging in applications involving large complex networks such as blockchain transactions and protein interaction graphs. The goal is to assess whether for $k=2,3$, $PD_k(\CG)$ is trivial or not. Indeed, such higher homological features might be associated, for example, with money laundering patterns or drug-target interactions.  

\begin{figure*}[h!]  \centering
	\begin{subfigure}{.39\textwidth}
		\centering
		\includegraphics[width=.99\linewidth]{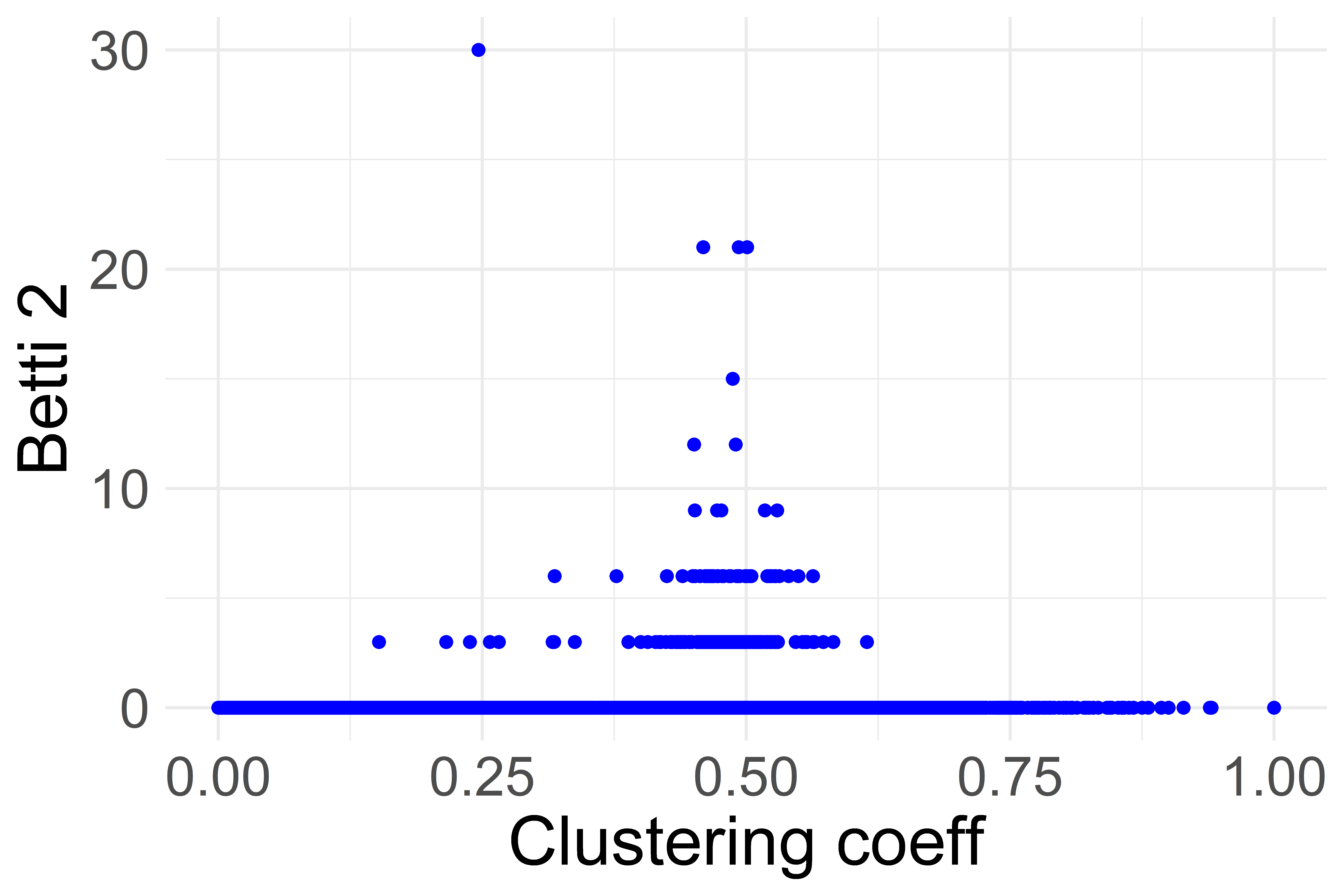}
	\end{subfigure}
	\begin{subfigure}{.39\textwidth}
		\centering
		\includegraphics[width=.99\linewidth]{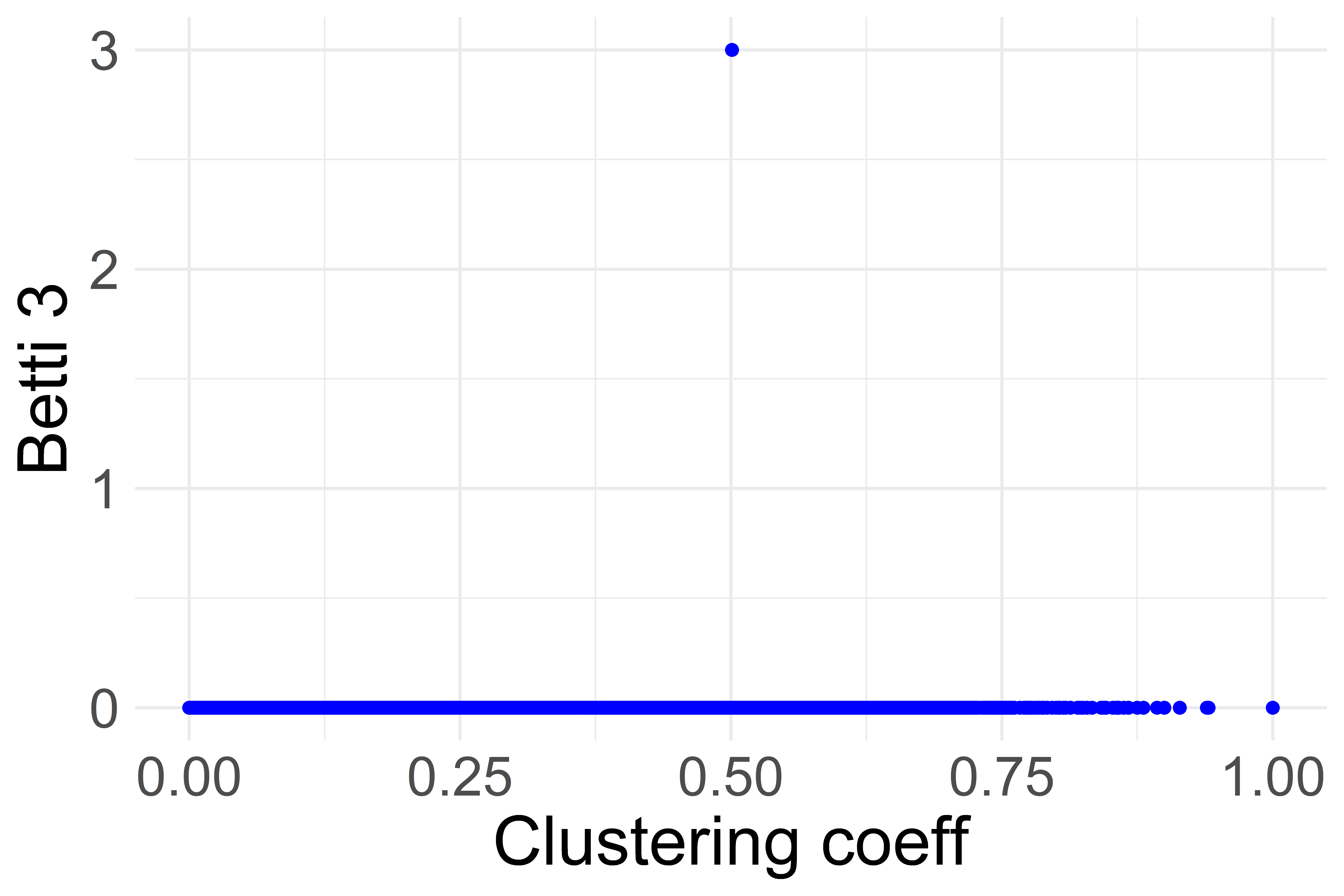}
	\end{subfigure}
	\caption{Clustering coefficients in kernel datasets. Betti 3 and higher do not exist in these graph datasets. In fact, even Betti 3 exists in a single graph only. Most Betti computations on such graphs can be avoided by using our conjecture on the clustering coefficient.}
	\label{fig:cluskernel}
\end{figure*}

% \begin{figure*}[ht]
% \begin{subfigure}{.24\textwidth}
%   \centering
%   \includegraphics[width=.99\linewidth]{figs/clusCoralSocial2.png}  \end{subfigure}
% \begin{subfigure}{.24\textwidth}
%   \centering
%   \includegraphics[width=.99\linewidth]{figs/clusCoralSocial3.png} 
% \end{subfigure}
% \begin{subfigure}{.24\textwidth}
%   \centering
%   \includegraphics[width=.99\linewidth]{figs/clusCoralSocial4.png}  \end{subfigure}
% \begin{subfigure}{.24\textwidth}
%   \centering
%   \includegraphics[width=.99\linewidth]{figs/clusCoralSocial5.png} 
% \end{subfigure}
% \caption{\footnotesize Clustering coefficients vs. number of topological features in Facebook and Twitter datasets. Each data point is a graph instance. We observe hundreds of higher topological features in these datasets which could not have been computed before because of the computational costs.}
% \label{fig:clussocial}
% \end{figure*}

Our experiments show that many real-life data sets tend to exhibit nontrivial second and third persistence diagrams (see~\Cref{fig:clussocial,fig:cluskernel}). We compare these datasets by considering their clustering coefficients $CC(\CG)$. Our observations indicate that when $CC(\CG)$ is too low or too high, then the higher-order persistence diagrams associated with these data tend to be trivial. Intuitively, when $CC(\CG)$ is too low, the graph $\CG$ tends to be too sparse to form a $k$-cycle and, hence, is unlikely to produce a nontrivial $PD_k(\CG)$. When $CC(\CG)$ is too high, then the graph $\CG$ is too dense and every $k$-cycle is filled immediately by a $(k+1)$-complex in $\wh{\CG}$. We formulate this phenomenon in the form of the following conjecture.

\noindent {\bf Conjecture:} For $k\geq 2$, there are $0<\alpha_k<\beta_k<1$ such that when $CC(\CG)>\beta_k$ and $CC(\CG)<\alpha_k$, then $PD_k(\CG)= \emptyset$ with high probability.

However, clustering coefficient has the computational complexity of $O(\left|\E\right|^{1.48})$~\cite{boot2016algorithms}, hence the using the coefficient may not seem as an improvement over coral reduction, which has a complexity of $O(\left |\E\right |+\left|\V\right|)$. However, note that the clustering coefficient can be iteratively computed as the average of vertex clustering coefficients (in a parallel setting). In this approach, a stopping condition can be applied to terminate early when the coefficient can be approximated. The gain in computational time may help the coral technique on extremely large graphs or their induced subgraphs.

\section{Checklist}

\begin{enumerate}

	\item For all authors...
	\begin{enumerate}
		\item Do the main claims made in the abstract and introduction accurately reflect the paper's contributions and scope?
		\answerYes{See abstract for details.}
		\item Did you describe the limitations of your work?
		\answerYes{See Section~\cref{sec:experiments}}
		\item Did you discuss any potential negative societal impacts of your work?
		\answerYes{See Appendix~D.} %Our work does not have any negative societal impact.
		\item Have you read the ethics review guidelines and ensured that your paper conforms to them?
		\answerYes{}
	\end{enumerate}

	\item If you are including theoretical results...
	\begin{enumerate}
		\item Did you state the full set of assumptions of all theoretical results?
		\answerYes{Please see the proofs given in Appendix~C}.
		\item Did you include complete proofs of all theoretical results?
		\answerYes{} Please see the proofs given in Appendix.
	\end{enumerate}

	\item If you ran experiments...
	\begin{enumerate}
		\item Did you include the code, data, and instructions needed to reproduce the main experimental results (either in the supplemental material or as a URL)?
		\answerYes{}
		\item Did you specify all the training details (e.g., data splits, hyperparameters, how they were chosen)?
		\answerYes{} Please see Section~\ref{sec:experiments}.
		\item Did you report error bars (e.g., with respect to the random seed after running experiments multiple times)?
		\answerYes{} Our experiments report averages with deviations (See Figure~\ref{fig:combinedresults}).
		\item Did you include the total amount of compute and the type of resources used (e.g., type of GPUs, internal cluster, or cloud provider)?
		\answerYes{} The experimental setting is given in Section~\ref{sec:experiments}.
	\end{enumerate}

	\item If you are using existing assets (e.g., code, data, models) or curating/releasing new assets...
	\begin{enumerate}
		\item If your work uses existing assets, did you cite the creators?
		\answerYes{}
		\item Did you mention the license of the assets?
		\answerYes{}
		\item Did you include any new assets either in the supplemental material or as a URL?
		\answerNA{}
		\item Did you discuss whether and how consent was obtained from people whose data you're using/curating?
		\answerNA{}
		\item Did you discuss whether the data you are using/curating contains personally identifiable information or offensive content?
		\answerNA{} We use graph data only. Our datasets do not contain personally identifiable information or offensive content.
	\end{enumerate}

	\item If you used crowdsourcing or conducted research with human subjects... 
	\begin{enumerate}
		\item Did you include the full text of instructions given to participants and screenshots, if applicable?
		\answerNA{}
		\item Did you describe any potential participant risks, with links to Institutional Review Board (IRB) approvals, if applicable?
		\answerNA{}
		\item Did you include the estimated hourly wage paid to participants and the total amount spent on participant compensation?
		\answerNA{}
	\end{enumerate}

\end{enumerate}

\section{Broader Impact} \label{sec:broader}

The proposed methodology makes an important step toward addressing the major existing roadblock in bringing the powerful machinery of TDA to the analysis of large-scale networks, from online social media to gene-to-gene interactions in bioinformatics. Undoubtedly, the application of TDA in learning such large-scale networks will have a substantial positive impact in a broad range of applications such as classification of anomalous subgraphs in blockchain transaction networks, discovering links between side effects and drugs, and drug re-purposing.
However, the critical negative impact of the proposed methodology is associated with our current inability to accurately quantify the loss of topological information due to pruning and the influence of such loss on the final learning task. This is a fundamental question that needs to be addressed in the future.

\end{document}